\documentclass[final, 10pt, DIV=12]{scrartcl}

\RequirePackage[
   natbib=true,
  style=authoryear,
  backend=biber,
  giveninits=true,
  uniquename=init
  ]{biblatex}

\renewcommand{\cite}[1]{\citep{#1}}

\usepackage{enumitem}
\usepackage{algorithm}
\usepackage{algpseudocode}
\usepackage{csquotes}
\usepackage{amsmath}
\usepackage{tikz}
\usetikzlibrary{positioning, matrix}

\usepackage{amssymb}        
\usepackage{amsthm}

\newif\ifarxivversion%
\arxivversiontrue%

\makeatletter
\def\@makechapterhead#1{
  {\parindent \z@ \raggedright \normalfont
   \Huge\bfseries \thechapter\quad #1
   \par\nobreak
   \vskip 20\p@
}}
\def\@makeschapterhead#1{
  {\parindent \z@ \raggedright \normalfont
   \Huge\bfseries #1
   \par\nobreak
   \vskip 20\p@
}}
\makeatother

\ifarxivversion%
   \theoremstyle{plain}
   \newtheorem{lemma}{Lemma}
   \newtheorem{theorem}{Theorem}
   \theoremstyle{definition}
   \newtheorem{definition}{Definition}
\else
   \theoremstyle{acmdefinition}
\fi
\newtheorem*{rem}{Remark}
\newtheorem{claim}{Claim}

\ifcsname DeclareNewFloatType\endcsname
\DeclareNewFloatType{listing}{}
\floatname{listing}{Program}
\fi

\def\ST{\ensuremath{\mathit{TRUE}}}
\def\SI{\ensuremath{\mathit{INIT}}}

\def\FORGET{\ensuremath{\mathit{FORGET}}}
\def\JOIN{\ensuremath{\mathit{JOIN}}}
\def\CTX{\ensuremath{\mathit{CTX}}}
\DeclareMathOperator{\CNF}{\mathit{CNF}}
\def\KT{\ensuremath{\mathit{KT}}}

\def\MSO{\ensuremath{\mathit{MSO}_2}}

\def\Adj{\ensuremath{\mathit{Adj}}}
\def\Edge{\ensuremath{\mathit{Edge}}}
\def\Nbr{\ensuremath{\mathit{Nbr}}}
\def\false{\ensuremath{\mathit{false}}}
\def\true{\ensuremath{\mathit{true}}}

\def\Apply{\ensuremath{\mathit{Apply}}}

\def\vert{\ensuremath{\mathrm{\mathbf{vert}}}}

\DeclareMathOperator{\nodedp}{\mathit{Node\_decision\_procedure}}
\DeclareMathOperator{\getAllCons}{\mathit{Get\_all\_possible\_assignments}}
\DeclareMathOperator{\tw}{\mathit{tw}}
\DeclareMathOperator{\pw}{\mathit{pw}}
\DeclareMathOperator{\node}{\mathit{node}}

\newcommand{\kvo}{\ensuremath{k_{\mathit{vo}}}}
\newcommand{\keo}{\ensuremath{k_{\mathit{eo}}}}
\newcommand{\kvs}{\ensuremath{k_{\mathit{vs}}}}
\newcommand{\kes}{\ensuremath{k_{\mathit{es}}}}
\newcommand{\domPhiG}{\ensuremath{\mathcal{D}_{\varphi, G}}}
\newcommand{\funcPhiG}{\ensuremath{\mathcal{F}_{\varphi, G}}}

\ifx\citet\undefined\else

\DeclareDelimFormat[textcite]{multinamedelim}{\addcomma\space}
\DeclareDelimFormat[textcite]{finalnamedelim}{\space and~}

\fi

\usepackage{orcidlink}
\title{Parameterized Complexity Of Representing Models Of MSO Formulas}

\author{Petr Kučera\,\orcidlink{0000-0002-7512-6260}\thanks{
  Department of Theoretical Computer Science and Mathematical
  Logic, Faculty of Mathematics and Physics, Charles University,
  Prague,
  Czech Republic,
  ORCID:\ 0000{-}0002{-}7512{-}6260,
  kucerap@ktiml.mff.cuni.cz
   }
   \and
Petr Martinek\thanks{
  Department of Theoretical Computer Science and Mathematical
  Logic, Faculty of Mathematics and Physics, Charles University,
  Prague,
  Czech Republic,
  petrmartinek4@seznam.cz
}
}

\begin{document}

\maketitle

\begin{abstract}
   Monadic second order logic (\MSO{}) plays an important role in parameterized
   complexity due to the Courcelle's theorem. This theorem states that the
   problem of checking
   if a given graph has a property specified by a given \MSO{} formula
   can be solved by a parameterized linear time algorithm with respect to the
   treewidth of the graph and the size of the formula.
   We extend this result
   by showing that models of \MSO{} formula with free variables can be
   represented with a decision diagram whose size is parameterized linear in
   the above mentioned parameter. In particular, we show a parameterized linear
   upper bound on the size of a sentential decision diagram (SDD) when treewidth
   is considered and a parameterized linear upper bound on the size of an
   ordered binary decision diagram (OBDD) when considering the pathwidth in the
   parameter. In addition, building on a lower bound on the size of OBDD
   by~\citet{razgon2014obdds}, we show that there is an \MSO{} formula and a class of graphs with
   bounded treewidth which do not admit an OBDD with the size parameterized by
   the treewidth. Our result offers a new perspective on the Courcelle's theorem
   and connects it to the area of knowledge representation.
\end{abstract}


\section{Introduction}

Monadic second order (\MSO{}) logic is a second-order logic in which the
second-order quantification is limited to quantifying over set variables. It is
suitable for formulating graph properties where sets may represent subset of
vertices or edges. Courcelle's theorem states, that graph properties represented
using \MSO{} formulas can be decided for any graph using a tree automaton in
linear time and using a state space that has a constant size, assuming that the
graph has bounded treewidth and that the \MSO{} formula has bounded
size~\cite{COURCELLE199012}. This makes \MSO{} a convenient tool for showing
fixed parameterized tractability of problems that can be modeled with
graphs~\cite{FG06}.

\citet{OKM23} considered the problem of representing the models of a \MSO{}
formula \(\varphi\) on a given graph \(G=(V, E)\) by an ordered binary decision
diagram (OBDD). The idea was to construct OBDDs for the atomic formulas and then
compose them by applying logical connectives and quantification on the OBDDs.
They checked this approach experimentally on concrete formulas. In the light of
Courcelle's theorem, it is natural to ask if the fixed parameter tractability
applies also to this approach and, in particular, if the size of the resulting
OBDD is fixed parameter linear with parameter that includes the treewidth of the
graph and the size of the formula. We show that this is true for the pathwidth,
but for the treewidth, we need to consider a different type of decision
diagrams. Unlike~\cite{OKM23}, we use a top-down approach in our construction
which follows a dynamic algorithm for deciding satisfiability of a \MSO{}
formula, similar to the one presented by~\citet{COURCELLE199012}.

OBDDs were introduced by~\citet{Bryant86} as a tractable
data structure which allows efficient answering of queries such as consistency
checking, model counting and enumeration, and also efficient combination of two
OBDDs with an arbitrary logical connective. Since their introduction, OBDDs have
been widely used in hardware verification~\cite{CM06} and model
checking~\cite{CGP99}.

OBDDs were used for the task of checking the satisfiability of an \MSO{} formula
by~\cite{henriksen1995mona} as part of Mona model checker. They served as a
concise way of representing the transition function of the underlying automaton.
Another \MSO{} solver was introduced in~\cite{LRRS12}. A practical approach for
solving \MSO{} formulas modeling problems such as vertex cover and dominating
set was considered by~\citet{KL09}.
\citet{LMM18} considered using QBF formulas
as an alternative representation for several problems modeled by \MSO{} formula.

Several types of representations generalizing OBDDs have been introduced in the
area of knowledge compilation (\citet{darwiche2002knowledge} give a comprehensive list). The
goal was to obtain a type of representation that is more succinct than OBDDs
(i.e., offers possibly smaller representation of the same function), while
preserving their good properties with respect to query answering and
transformations. Among these representation languages, sentential decision
diagrams (SDD) were introduced by~\citet{darwiche2011sdd}. They have similar
properties as OBDDs with respect to queries and transformations while being
strictly more succinct as shown by~\citet{B16}. One of the main differences
between OBDDs and SDDs is in their structure. While OBDDs follow a fixed linear
order of variables, SDDs have a tree-like structure based on the notion of a
v-tree. As such, SDDs are more suitable for representing models of an \MSO{}
formula on a graph parameterized by its treewidth.

This is also reflected by our results. In particular, we show a fixed parameter
linear upper bound on the size of an SDD representing models of an \MSO{}
formula \(\varphi\) on a graph \(G\) where the parameter is the sum of the
treewidth of \(G\) and the length of \(\varphi\). We show a similar result for
OBDDs where we consider pathwidth as part of the parameter instead of the
treewidth.
In addition, we use the lower
bound of~\citet{razgon2014obdds} to prove that a bound parameterized by
treewidth cannot be provided for OBDDs. The difference in the structure between
OBDDs and SDDs thus cannot be overcome.

Our result is also connected to other parameterized bounds on the sizes of SDDs and OBDDs.
It is known that if the primal graph a conjunctive normal form (CNF) formula has
treewidth bounded by \(w\), then there is an equivalent SDD of size \(O(n2^w)\) where \(n\)
is the number of variables~\cite{darwiche2011sdd}. A similar bound is known for OBDDs and
pathwidth and these bounds can be generalized to circuits as well~\cite{AMS18}.
By~\cite{razgon2014obdds}, size of an OBDD cannot be parameterized by the
treewidth of the primal graph of the input CNF formula.

Finally, let us note that both OBDDs and SDDs represent boolean functions while
variables in an \MSO{} formula \(\varphi\) represent vertices, edges, or sets of
vertices or edges of a given graph. To represent a single \MSO{} variable \(x\),
we thus use several propositional variables to represent the possible
assignments to \(x\). To have these propositional encodings of models in
one-to-one correspondence with the actual models of \(\varphi\), we must filter
out the propositional assignments that do not correspond to a valid model of
\(\varphi\).

Our paper is organized as follows. We start by introducing the notation and
definitions needed to present our results in Section~\ref{sec:background}.
Section~\ref{sec:results} gives a formal statements of our results. We then
describe a dynamic programming procedure for deciding an \MSO{} formula in
Section~\ref{sec:decproc}. We use this procedure in the constructions showing
upper bounds on the sizes of SDDs (Section~\ref{sec:sdd-ub}) and OBDDs
(Section~\ref{sec:obdd-ub}). We give a lower bound on the size of OBDDs when
treewidth is used as part of the parameter in Section~\ref{sec:obdd-lb}.
Section~\ref{sec:notes} contains additional notes on the constructions which
include a concrete upper bound on the size of a decision diagram for a formula in prenex
form. Finally, Section~\ref{sec:conclusion} contains a few closing
remarks and directions for future research.

%


\section{Background}\label{sec:background}

In this chapter, we will introduce the notions and notation we shall use to
present and prove our results. 

\subsection{Graphs And Their Decompositions}

We assume that the reader is familiar with the basics of the graph theory. For
reference, see for instance~\cite{B98}. In this section, we shall recall the
notions of tree and path decompositions and their widths.

A \emph{tree decomposition} of an undirected graph \(G=(V, E)\) is defined as a
triple
\(T=(V_T, E_T, \ell)\) where graph \((V_T, E_T)\) is a tree
and \(\ell:V_T\to 2^V\) is a labeling function which satisfies the following
properties.

\begin{enumerate}[label={(T\arabic*)}]
   \item\label{t2} For every edge $\{v, w\} \in E$ there is a node $x \in V_T$
      for which $\{v,w\} \subseteq \ell(x)$.
   \item\label{t3} For $v \in V$, let $T^{(v)}$ be the subgraph of \(T\) induced
      by the set of nodes $\{x \in V_T \mid v \in
      \ell(x) \}$. Then for each $v \in V(G)$, $T^{(v)}$ is non-empty and
      connected.
\end{enumerate}

The \emph{width} of tree decomposition 
\(T\) is defined as
\(\max_{x\in V_T} |\ell(x)|-1\) (i.e., the maximum size of a label set
subtracted by \(1\) to have treewidth of a tree equal to \(1\)). The \emph{treewidth} of \(G\), denoted as
\(\tw(G)\), is then the minimum possible width of a tree decomposition of \(G\).

If tree \((V_T, E_T)\) is a path (i.e., all nodes have degree at most \(2\)),
then \(T\) is a \emph{path decomposition} of \(G\). The \emph{pathwidth} of
\(G\), denoted as \(\pw(G)\) is the minimum width of any path decomposition
of \(G\).

To simplify the presentation of our algorithm, we shall consider nice
decompositions. A \emph{nice tree decomposition} of \(G\) is a quadruple \(T=(V_T, E_T,
\ell, r)\) where \((V_T, E_T)\) is a rooted binary tree with root \(r\in V_T\)
and \(\ell:V_T\to 2^V\) is a labeling function satisfying the following
properties.

\begin{enumerate}[label={(N\arabic*)}]
   \item\label{n1} $(V_T, E_T, \ell)$ is a tree decomposition of $G$
   \item\label{n2} $\ell(r) = \emptyset$
   \item\label{n3} Each node \(s\in V_T\) has at most two children and in addition:
      \begin{enumerate}[leftmargin=2em, label={(N3.\arabic*)}]
         \item\label{n31} If $s$ has a single child $t$, then the size of
            symmetric difference of $\ell(s)$ and $\ell(t)$ is exactly one.
         \item\label{n32} If $s$ has two children $t$ and $u$, then $\ell(s) = \ell(t) = \ell(u)$.
      \end{enumerate}
\end{enumerate}

If \(T=(V_E, E_T, \ell, r)\) is a nice tree decomposition of graph \(G=(V, E)\),
then each node \(s\in V_T\) has one of the following types.

\begin{description}
   \item[Leaf] node $s$ has zero child nodes.
   \item[\enquote{Introduce \(v\)}] node $s$ has a single
      child node $t$, such that $\ell(s) = \ell(t) \cup \{v\}$ for vertex \(v\in
      V\).
   \item[\enquote{Forget \(v\)}] node \(s\) has a single
      child node $t$, such that $\ell(s) = \ell(t) \setminus \{v\}$ for vertex \(v\in V\).
   \item[Join] node $s$ has two child nodes \(t\) and \(u\), such that
      \(\ell(s)=\ell(t)=\ell(u)\).
\end{description}

Nice tree decomposition was introduced by \citet{kloks1994treewidth} who shows
that if \(G\) has \(n\) vertices, then there is a nice tree decomposition of
\(G\) of width \(\tw(G)\) with \(4n\) nodes. Condition~\ref{n2} was not included
in the definition by~\cite{kloks1994treewidth}, we can satisfy it by adding at
most \(\tw(G)\leq n-1\) forget nodes on top of the root. Thus we obtain the
following lemma.

\begin{lemma}\label{strongnice}
   Let $G=(V, E)$ be an undirected graph on $n>0$ vertices and with treewidth
   of $w$. Then $G$ has a nice tree decomposition $T$ of width $w$, containing at
   most $5n$ nodes.
\end{lemma}

Proof of our main result relies on a dynamic programming algorithm which follows
the structure of a nice tree decomposition proceeding from the leaves to the
root. Most work is done when processing the forget nodes using the fact that
each node is forgotten exactly once.

\begin{lemma}\label{lemma:oneforget}
   Let $T=(V_T, E_T, \ell, r)$ be a tree decomposition of an undirected graph 
   $G=(V, E)$ and let $v \in V$ be a vertex.
   Then \(V_T\) contains exactly one node of \enquote{forget $v$} type.
\end{lemma}

Formulas in \MSO{} treat edges as elementary objects similar to vertices. When
constructing an SDD for the formula, we will thus need decision variables
associated with the edges as well. Decision variables related to a vertex \(v\)
are introduced when processing  \enquote{forget \(v\)} node. Similarly, decision
variables related to an edge \(\{u, v\}\) will be introduced when
processing the first of the two associated forget nodes --- \enquote{forget
\(u\)} or \enquote{forget \(v\)}. To this end, we define the notion of
forgetting an edge as follows.

\begin{definition}
   Let \(T=(V_T, E_T, \ell, r)\) be a nice tree decomposition of \(G=(V, E)\).
   We say that node \(s\in V_T\) forgets an edge
   \(\{u, v\}\) if \(s\) is a forget node with a single child \(t\) and
   \(\{u, v\}\subseteq\ell(t)\) while \(\{u, v\}\not\subseteq\ell(t)\).
\end{definition}

The uniqueness of forget nodes holds for edges as well.

\begin{lemma}\label{lemma:oneforgete}
   Let $T=(V_T, E_T, \ell, r)$ be a nice tree decomposition of an undirected graph 
   $G=(V, E)$ and let $e \in E$ be an edge. Then there is exactly one node in $T$
   which forgets \(e\).
\end{lemma}

\begin{proof}
   Denote the endpoints of $e$ as $u$ and $v$. By Lemma~\ref{lemma:oneforget},
   there is exactly one \enquote{forget $u$} node and one \enquote{forget $v$}
   node. By tree decomposition property~\ref{t2}, there is a node $z$ with label
   containing both $u$ and $v$. Consider the path from $z$ to root. Both
   \enquote{forget $u$} node  and \enquote{forget $v$} node lie on this path. These
   nodes are also distinct by the definition of node types. Consider the sequence
   of the nodes along the branch, starting at $z$. Either \enquote{forget $u$} node
   or \enquote{forget $v$} node will appear first along this path. Without loss of
   generality, let \enquote{forget $u$} be the first node. The node \enquote{forget
   $u$} is the unique node that forgets $e$, because all the other nodes on path from
   this node to the root do not have $u$ in their label.
\end{proof}

\subsection{Monadic Second Order Logic Of Graphs}\label{ssec:mso}

Monadic second order logic (MSO) is a fragment of second order logic, in which
quantification over relational variables is limited to unary relations. Since a
unary relation is defined as a subset of the universe, we refer to unary
predicates simply as sets in the context of MSO\@. Analogically we refer to
quantification over relations as quantification over sets and relational
variables are called set variables.

\MSO{} is the name given to monadic second order logic over the language of
graphs, in which we can quantify both over vertices and edges.
We use the definition of \MSO{} which is based on the one used by
\citet{COURCELLE199012}. To simplify the presentation of our results, we only
use a minimal complete set of connectives (conjunction and negation) and we only
consider existential quantification. The rest of the connectives and universal
quantification can be then derived using the usual semantic of predicate logic.
In addition, instead of predicate \Edge{}, we use \Adj{}, since in the classes
of graphs that we study, both of these predicates have the same expressive
power, while \Adj{} is simpler to handle. See Section~\ref{sec71}
for notes on using our results with alternative definitions of
\MSO{}.

To simplify reasoning in and about this logic, we consider the underlying
universe $U$ to be composed of two disjoint sets, $U_V$ (representing vertices)
and $U_E$ (representing edges). Let us note that we could also consider a single
universe and use additional unary predicates \textit{is\_vertex} and
\textit{is\_edge} to differentiate between vertices and edges. Formulas in the
definition with the sub-universes can be translated into an equivalent formula
in the definition with the single universe under the assumption that $is\_vertex
= U_V$ and $is\_edge = U_E$.

The language of \MSO{} consists of the following symbols:
\begin{description}
   \item[Countably many variable symbols.] Each variable is either an object variable, or a set
      variable. Each variable in \MSO{} also has a \textbf{sort} which
      restricts assignments to this variable. The sort of a variable is either the
      \emph{vertex sort} or the \emph{edge sort}. In total, there are four types of variables:
      \begin{description}
         \item[Vertex object variables] which can have assigned a single element from $U_V$.
         \item[Edge object variables] which can have assigned a single element from $U_E$.
         \item[Vertex set variables] which can have assigned a subset of $U_V$.
         \item[Edge set variables] which can have assigned a subset of $U_E$.
      \end{description}
   \item[Logical symbols] $\neg$ and $\wedge$.
   \item[Equality] ($=$) and \textbf{membership} ($\in$).
   \item[Four types of existential quantifiers] $(\exists x \in U_V)$, $(\exists x \in U_E)$, $(\exists X \subseteq U_V)$, and $(\exists X \subseteq U_E)$.
   \item[Non-logical binary relation] $\Adj$.
\end{description}

The grammar of \MSO{} allows the construction of following formulas:
\begin{itemize}
   \item $\Adj(v, e)$, where $v$ is a vertex object variable and $e$ is an edge object variable.
   \item $(x = y)$, where $x$ and $y$ are object variables of the same sort.
   \item $(x \in Y)$, where $x$ is an object variable and $Y$ is a set variable, both of the same sort.
   \item $\neg \psi$, where $\psi$ is an \MSO{} formula.
   \item $\psi \wedge \rho$, where $\psi$ and $\rho$ are \MSO{} formulas.
   \item $(\exists x \in U_V)\psi$, where $\psi$ is an \MSO{} formula and $x$ is a vertex object variable which is free in $\psi$.
   \item $(\exists x \in U_E)\psi$, where $\psi$ is an \MSO{} formula and $x$ is an edge object variable which is free in $\psi$.
   \item $(\exists X \subseteq U_V)\psi$, where $\psi$ is an \MSO{} formula and $X$ is a vertex set variable which is free in $\psi$.
   \item $(\exists X \subseteq U_E)\psi$, where $\psi$ is an \MSO{} formula and $X$ is an edge set variable which is free in $\psi$.
\end{itemize}

We use the usual semantic of predicate logic to evaluate a
formula.

\begin{definition}[Evaluation of \MSO{} formula]
   Let $\varphi$ be an \MSO{} formula, let $G=(V, E)$ be a graph and let
   $\alpha$ be an assignment, which assigns values to all free variables of
   $\varphi$, by taking values from universes $U_V := V$ and $U_E := E$. We
   evaluate $\varphi$ on graph $G$ for assignment $\alpha$ by recursion on
   structure of $\varphi$ as follows:
   \begin{itemize}
      \item $\varphi \equiv \Adj(v,e)$ is true for $\alpha$ if and only if vertex \(\alpha(v)\) belongs to edge \(\alpha(e)\).
      \item $\varphi \equiv (x = y)$ is true for $\alpha$ if and only if $\alpha(x) = \alpha(y)$.
      \item $\varphi \equiv (x \in Y)$ is true for $\alpha$ if and only if $\alpha(x) \in \alpha(Y)$.
      \item $\varphi \equiv \neg \psi$ is true for $\alpha$ if and only if $\psi$ is false.
      \item $\varphi \equiv \psi \wedge \rho$ is true for $\alpha$ if and only if $\psi$ is true and $\rho$ is true.
      \item $\varphi \equiv (\exists x \in U_V)\psi$ is true for $\alpha$, if
         and only if there exists some $v \in U_V$, such that $\psi$ is true for
         $\alpha \cup \{x := v\}$.
      \item $\varphi \equiv (\exists x \in U_E)\psi$ is true for $\alpha$, if
         and only if there exists some $e \in U_E$, such that $\psi$ is true for
         $\alpha \cup \{x := e\}$.
      \item $\varphi \equiv (\exists X \subseteq U_V)\psi$ is true for $\alpha$,
         if and only if there exists some $W \subseteq U_V$, such that $\psi$ is
         true for $\alpha \cup \{X := W\}$.
      \item $\varphi \equiv (\exists X \subseteq U_E)\psi$ is true for $\alpha$,
         if and only if there exists some $C \subseteq U_E$, such that $\psi$ is
         true for $\alpha \cup \{X := C\}$.
   \end{itemize}
\end{definition}

The size of the formula is defined as the size of its syntax tree while ignoring
parentheses and counting a quantifier as a single symbol. More concretely:

\begin{definition}[Size of a formula]
   For an \MSO{} formula $\varphi$, we define its size $|\varphi|$ recursively by its structure as
   \begin{itemize}
      \item $|\Adj(v,e)| = 3$
      \item $|(x = y)| = 3$
      \item $|(x \in Y)| = 3$
      \item $|\neg \psi| = 1 + |\psi|$
      \item $|\psi \wedge \rho| = 1 + |\psi| + |\rho|$
      \item $|(\exists x \in U_V)\psi| = 1 + |\psi|$, similarly for other variations of quantification.
   \end{itemize}
\end{definition}

One key observation is that the number of all variable symbols that are
part of atomic subformulas in any formula $\varphi$, is upper bounded by
$|\varphi|$. More importantly, the number of unique free variables is also
bounded by $|\varphi|$. 

\subsection{Decision Diagrams}

Our results give parameterized upper bounds on the sizes of decision diagrams
representing solutions to a graph problem specified by a given \MSO{} formula
\(\varphi\).
In this section, we shall describe the two classes of decision diagrams we
consider in our work --- binary decision diagrams and sentential decision
diagrams. These decision diagrams represent boolean
functions defined over propositional variables.
The variables of \(\varphi\), on the other hand, have much bigger domains. A
single variable of \(\varphi\) is represented with a number of propositional
variables which together specify its value. We shall describe this in more
details in Section~\ref{sec:results}.

\subsubsection{Binary Decision Diagrams}

Ordered binary decision diagram is an efficient data structure for representing
boolean functions which was introduced by \citet{Bryant86}. Definitions below
are adapted from their work. We are specifically interested in reduced ordered
binary decision diagrams, as they are both smallest and unique for a given
function.

\begin{definition}[Binary decision diagram]
   A \textbf{binary decision diagram} (BDD) is a rooted directed acyclic
   graph satisfying the following conditions.
   \begin{itemize}
      \item Exactly one node (the \textbf{root node}) has in-degree \(0\).
      \item Inner nodes (called \textbf{decision nodes}) have out-degree \(2\) and
         unbounded in-degree. One of the outgoing edges is labeled with \(0\) and
         the other one is labeled with \(1\). The node itself is labeled with a
         decision variable.
      \item \textbf{Leaf nodes} are labeled with \(0\) or \(1\).
   \end{itemize}
\end{definition}

Let $A$ be a BDD with root vertex $r$. The value of \(A\) on an assignment
\(\delta\) to decision variables of \(A\) (denoted \(A(\delta)\)) is obtained in
the following way.
\begin{itemize}
   \item If $r$ is a leaf node, then \(A(\delta)\) is the label of $r$.
   \item If $r$ is a decision node labeled with decision variable $x$, we follow
      the edge labeled by $\delta(x)$ to a new node $v$ and recursively evaluate
      the subgraph rooted at $v$.
\end{itemize}

We say that $A$ \textbf{computes} boolean function $f$, if the set of decision
variables in $A$ is the subset of the variables of $f$ and for each
$f$-assignment $\delta$, $f(\delta) = A(\delta)$. Note that if the root \(r\) is
a leaf node, then \(A\) represents a constant function with function value given
by the label of \(r\).

\begin{definition}[Ordered binary decision diagram]
   Let $<$ be a strict linear ordering on the set of decision variables. A decision
   diagram $A$ is said to \textbf{respect} $<$, if \(x_c<x_d\) for every node
   \(c\) and its descendant \(d\) labeled with variables \(x_c\) and \(x_d\)
   respectively.
   If a binary decision diagram respects any ordering of
   decision variables, we call it an \textbf{ordered binary decision diagram}
   (OBDD).
\end{definition}

\begin{figure}
   \includegraphics[width=\linewidth]{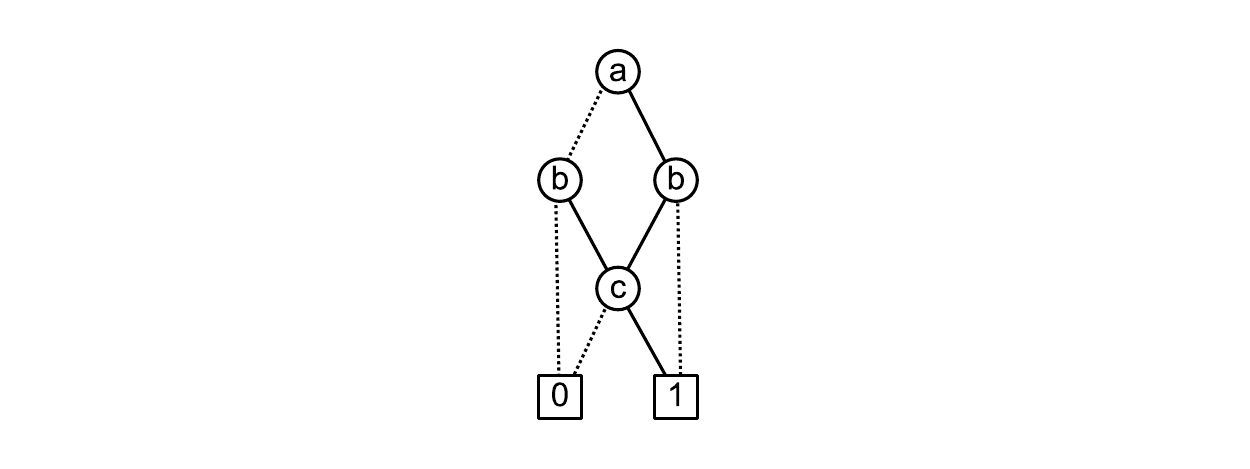}
   \caption{An example of an OBDD that computes function $(a \wedge \neg b) \vee (b \wedge c)$.} \label{fig:obdd} 
\end{figure}

Figure~\ref{fig:obdd} shows an example of an OBDD, which respects variable
ordering $a<b<c$. Edges labeled with \(0\) are represented by dotted lines, while
edges labeled with \(1\) are represented by full lines. Edges are directed downwards.
The OBDD in the figure computes boolean function $(a \wedge \neg b) \vee (b
\wedge c)$.

\begin{rem}
   One of the interesting properties of OBDDs proved by \citet{Bryant86} is
   the canonicity of reduced OBDDs. An OBDD is reduced, if and only if for every
   decision node, its two outgoing edges lead to different nodes and for every
   pair of distinct decision nodes, the functions computed by the sugraphs
   rooted at the decision nodes are distinct. For a fixed variable ordering and
   boolean function $f$, there is a unique reduced OBDD which computes $f$. The
   reduced OBDD has optimal size for given variable ordering, which can be
   useful when proving bounds on sizes. However, we will not need this notion to
   present our results, because we prove all bounds on the size directly, mostly
   relying on similar results for SDDs, which are the primary focus of this
   paper.
\end{rem}

%

\subsubsection{Sentential Decision Diagrams}

Sentential decision diagram (SDD) was introduced by \citet{darwiche2011sdd}.
Definitions below are adapted from their work. While OBDDs respect a linear
order of variables, decisions in SDD follow a tree-like structure called v-tree.
This fundamental difference relates OBBDs to pathwidth and SDDs to treewidth in
our results.

\begin{definition}[v-tree]
   A \textbf{v-tree} for a set of decision variables $\{v_1,\dots,v_k\}$ is a
   rooted tree \(T\) with the following properties:
   \begin{itemize}
      \item Each leaf is labeled with a decision variable and the leaves of \(T\) are
         in one to one correspondence with the decision variables.
      \item All edges are labeled with $L$ or $R$.
      \item All internal nodes have exactly two child nodes --- the \textbf{left
         child} accessed by edge
         labeled with $L$ and the \textbf{right child} accessed by edge labeled with $R$.
   \end{itemize}
   We will use $L(x)$ (resp.\ \(R(x)\)) to denote the left (resp.\ right) child of
   an inner node \(x\).
   The subtree rooted at \(L(x)\) will
   be called the \textbf{left subtree of $x$} and will be denoted as $x^L$.
   Analogically, we define the \textbf{right subtree of $x$} denoted as $x^R$.
\end{definition}

Unlike in OBDDs, the decisions in SDDs are made based on sentences which form a
partition.

\begin{definition}[Boolean partition]
   A set of boolean functions $\{p_1,p_2,\dots,p_k\}$ forms
   a \textbf{boolean partition}, if both of the following properties hold:
   \begin{itemize}
      \item $p_i \wedge p_j \equiv \false$ for $i \ne j$
      \item $\bigvee_{i=1}^k p_i \equiv \true$
   \end{itemize}
\end{definition}

We are now ready to define sentential decision diagrams.

\begin{definition}[Sentential decision diagram]
   A \textbf{sentential decision diagram (SDD) $\Delta$ respecting v-tree $u$} with the
   semantics in form of boolean function $\langle \Delta \rangle$
   is one of the following:
   \begin{itemize}
      \item $\Delta = \bot$ with semantics $\langle \bot \rangle \equiv \false$
      \item $\Delta = \top$ with semantics $\langle \top \rangle \equiv \true$
      \item $\Delta = X$ if $u$ is a leaf labeled by variable $X$. The semantics is $\langle X \rangle \equiv X$
      \item $\Delta = \neg X$ if $u$ is a leaf labeled by variable $X$. The semantics is $\langle \neg X \rangle \equiv \neg X$
      \item $\{(p_1,s_1),\dots,(p_k,s_k)\}$ where $p_1, \dots, p_n$ are SDDs
         respecting any node in subtree $u^L$, $s_1, \dots, s_k$ are SDDs respecting any node
         in subtree $u^R$, and the set $\{\langle p_1 \rangle,\dots,\langle p_k \rangle\}$ is
         a boolean partition. The semantics is:
         \begin{equation*}
            \langle \{(p_1,s_1),\dots,(p_k,s_k)\} \rangle \equiv
            \bigvee_{i=1}^k \langle p_i \rangle \wedge \langle s_i \rangle
         \end{equation*}
   \end{itemize}
   If an SDD has one of the first four forms, we will call it a \textbf{terminal}.
   If an SDD has the last form, we will call it a \textbf{decomposition}. In
   decompositions, we refer to $p_1, \dots, p_k$ as \textbf{primes} and to SDDs
   $s_1, \dots, s_k$ as \textbf{subs}.
\end{definition}

%

%

\begin{figure}
   \includegraphics[width=\linewidth]{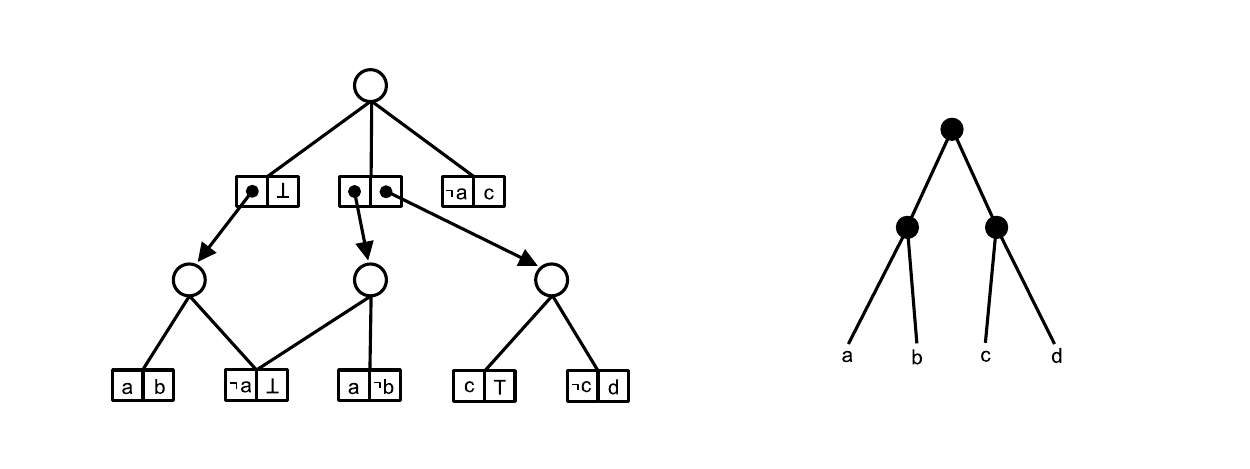}
   \caption{An example of an SDD (on the left) and a v-tree it respects (on the right). }
   \label{fig:sdd} 
\end{figure}

Figure~\ref{fig:sdd} shows an example of an SDD and a v-tree respected by the
SDD\@. A prime-sub pair is represented by two boxes next to each other. The left
box contains the prime, while the right box contains the sub. A prime or a sub
may be specified directly if it is a terminal, or by a pointer to a decomposition.
Decompositions are represented by circles and their prime-sub pairs are
represented by the connected box pairs. The topmost decomposition with three
prime-sub pairs is the root.
The SDD in the figure computes boolean
function $(a \wedge \neg b \wedge (c \vee d)) \vee (\neg a \wedge c)$.

We define the size of an SDD as the sum of the sizes of its decompositions with
terminal nodes having size \(1\). More formally:

\begin{definition}[Size of an SDD]
   Let $\Delta$ be an SDD\@. Let $C$ be the set of all distinct functions computed
   by the SDDs that are part of the recursive definition of $\Delta$, including
   $\Delta$ itself. We define its \textbf{size} $|\Delta|$ to be the sum of the
   following values:
   \begin{itemize}
      \item \(1\) for terminals
      \item $n$ for a decomposition represented by a set of size $n$.
   \end{itemize}
\end{definition}

Let us note that an OBDD can be considered as a special case of an SDD in which
the v-tree has form of a right path (i.e.\ the left child of each inner node is
a leaf) specifying the order of the veriables. Each decomposition then takes
form of a Shannon decomposition representing a decision node. SDDs have similar
properties to OBDDs with respect to their applications.
\citet{darwiche2011sdd}
also describes canonical SDDs which are uniquely defined for a given function
and v-tree. Unlike in the case of OBDDs, requiring canonicity may lead to an
exponential blow-up in size.

%

\subsection{Parameterized Analysis}

In parameterized analysis (see for instance~\cite{FG06}),
a parameter is defined for each instance of a problem
and we study the dependence of the complexity of the problem on the
parameter.

For example, consider a problem that has a graph as an instance.  While the size
of the instance only takes into account the number of vertices and edges in the
graph, the parameter can consider the internal structure of the graph. For
example, the parameter can be the treewidth of the graph, or the size of the
largest clique in the graph.

We will be using the following two notions related to the complexity of
parameterized problems.

\begin{definition}
   We say that a knowledge base $B$ has \textbf{fixed parameter linear size}, if it
   represents an input of size $n$ with parameter $k$ and if its size with respect
   to these variables is at most:
   \begin{equation*}
      |B| \in O(f(k) \cdot n)
   \end{equation*}
   for some computable function $f$.
\end{definition}

\begin{definition}
   We say that a knowledge base $B$ has \textbf{at least slice-wise polynomial
   size}, if for given $n$ and $k$, we can find a problem instance of size at most
   $n$ and parameter at most $k$, such that the minimal size of the knowledge base
   has lower bound of:
   \begin{equation*}
      |B| \in \Omega(n^{f(k)})
   \end{equation*}
   for some computable function $f$.
\end{definition}

An input of the problem we aim to solve in this paper has form of a graph $G$
and an \MSO{} formula $\varphi$. We shall consider parameter \(k=|\varphi|+w\)
where \(w\) is either the treewidth, or the pathwidth of \(G\) depending on the
context. The size of the input is technically $(|V(G)| + |E(G)| + |\varphi|)$,
however, since $|\varphi|$ is part of the parameter, we will work with $n :=
|V(G)| + |E(G)|$. The solution to the problem instance is a knowledge base and
we wish to bound its size. In two of our results, we will show a fixed parameter
linear upper bound on the size of a decision diagram for a given instance. In
the third result, we will show a slice-wise polynomial lower bound by
constructing suitable problem instances.


\section{Our Results}\label{sec:results}

A well-known result of~\citet{COURCELLE199012} shows that satisfiability of an
\MSO{} formula \(\varphi\) on a given graph \(G\) can be checked by an algorithm
whose complexity is fixed parameter linear in parameter \(k=|\varphi|+w\) where
\(w\) denotes the treewidth of \(G\). Our results build on that by showing that
the models of such formula on a given graph can be represented by a sentential
decision diagram whose size is fixed parameter linear. For pathwidth, we get a
similar result for OBDD and, in addition, we also show that the size of OBDD
cannot be parameterized by the treewidth.

We shall now give a formal specification of the problem we aim to solve and the
statements of the theorems with the main results of this paper.


Let $G=(V, E)$ be an undirected graph and $\varphi$ be an \MSO{} formula which is interpreted
over \(G\). Our goal is to
represent all models \(\alpha\) of $\varphi$ in form of a decision diagram that
only uses propositional variables. To this end, each variable \(x\) in \(\varphi\)
will be associated with a vector of propositional variables which represent the
possible assignments to \(x\).

\begin{definition}[Decision variables]\label{def:decision-variables}
Let $G=(V, E)$ be a graph and $\varphi$ be an \MSO{} formula which is
interpreted over \(G\). Consider the following variables: 

\begin{enumerate}
\item $[x = v]$ for each free vertex object variable $x$ in $\varphi$ and $v \in V(G)$
\item $[x = e]$ for each free edge object variable $x$ in $\varphi$ and $e \in E(G)$
\item $[v \in X]$ for each free vertex set variable $X$ in $\varphi$ and $v \in V(G)$
\item $[e \in X]$ for each free edge set variable $X$ in $\varphi$ and $e \in E(G)$
\end{enumerate}

We will denote the set of these variables as $\mathcal{D}_{\varphi, G}$.
\end{definition}

To differentiate the types of assignments related to the \MSO{} formula
\(\varphi\), we shall use \emph{\(\varphi\)-assignment} to refer to an
assignment to \MSO{} variables in \(\varphi\) and \emph{\(\domPhiG\)-assignment}
to refer to an assignment that assigns boolean values to variables in
\(\domPhiG\).

While the decision variables in $\mathcal{D}_{\varphi, G}$ can be used to encode
any $\varphi$-assignment $\alpha$, not all $\mathcal{D}_{\varphi,
G}$-assignments encode a valid assignment $\alpha$. Consider for example a
vertex object variable $x$ and vertices $u$ and $v$. Assigning $1$ to both
$[x=v]$ and $[x=u]$ would lead to an assignment which does not correspond to any
valid $\varphi$-assignment, as only one value can be assigned to any object
variable. We thus need the decision diagram to represent only consistent
assignments to \(\domPhiG\) which have a corresponding \(\varphi\)-assignment.

\begin{definition}[Consistent assignment]\label{def:consistent}
   Let $G=(V, E)$ be a graph and $\varphi$ be an \MSO{} formula which is
   interpreted over \(G\). 
   A binary assignment \(\delta\) of variables in \(\mathcal{D}_{\varphi, G}\)
   is \textbf{consistent}, if and only if \(\delta\) satisfies the following two
   conditions.
   \begin{enumerate}
      \item For every vertex object variable \(x\), \(\delta\) satisfies exactly one variable among
         \([x=u], u\in V\).
      \item For edge object variable \(x\), \(\delta\) satisfies exactly one variable among
         \([x=e], e\in E\).
   \end{enumerate}
\end{definition}

The number of decision variables introduced by
Definition~\ref{def:decision-variables} can be bounded as follows.

\begin{lemma}[Number of decision variables]%
\label{ndvar}
Let $G=(V, E, \vert)$ be a graph and $\varphi$ an \MSO{} formula which is
interpreted over \(G\). Let $n = |V| + |E|$. The number of the decision variables in
$\mathcal{D}_{\varphi, G}$ is at most $|\varphi| \cdot n$.
\end{lemma}
\begin{proof}
   Each free variable in \(\varphi\) contributes at most $n$ decision variables
   by Definition~\ref{def:decision-variables}, so together we have at most
   \(|\varphi|\cdot n\) decision variables.
\end{proof}

Models of \(\varphi\) can now be represented by the following boolean function.

\begin{definition}\label{def:phi-func}
   Let \(G=(V, E)\) be a graph and \(\varphi\) be an \MSO{} formula which is
   interpreted over \(G\). Then we define function \(\mathcal{F}_{\varphi,
   G} : \mathcal{D}_{\varphi, G}\to\{0, 1\}\) as the function whose models are
   exactly the consistent assignments \(\delta\) which represent the models of
   \(\varphi\).
\end{definition}

In this paper, we aim to solve the following problem: Given graph
\(G=(V, E)\) and \MSO{} formula \(\varphi\) which is interpreted over \(G\),
find a succinct representation of \(\mathcal{F}_{\varphi, G}\) in a given
representation language. Depending on the representation language (OBDD, SDD)
and the parameter of \(G\) (treewidth, pathwidth) we consider, we get three
propositions we prove in the rest of the paper.

\begin{theorem}\label{thm:sdd-ub}
   Let \(G=(V, E)\) be a graph with treewidth \(w\). Let \(\varphi\) be
   an \MSO{} formula interpreted over \(G\). Let \(n=|V|+|E|\) and
   \(k=w+|\varphi|\). Then there exists an SDD representing
   \(\mathcal{F}_{\varphi, G}\) of size that is fixed-parameter
   linear in \(n\) for parameter \(k\)
\end{theorem}

\begin{theorem}\label{thm:obdd-ub}
   Let \(G=(V, E)\) be a graph with pathwidth \(w\). Let \(\varphi\) be
   an \MSO{} formula interpreted over \(G\). Let \(n=|V|+|E|\) and
   \(k=w+|\varphi|\). Then there exists an OBDD representing
   \(\mathcal{F}_{\varphi, G}\) of size that is fixed-parameter
   linear in \(n\) for parameter \(k\)
\end{theorem}

\begin{theorem}\label{thm:obdd-lb}
   For every \(w\geq 1\), there is an \MSO{} formula \(\varphi\), an infinite
   sequence of graphs \(G_1, G_2, \dots\), and a computable function \(f\)
   satisfying the following properties:
   \begin{itemize}
      \item All graphs \(G_i\) have treewidth at most \(w\).
      \item For every graph \(G_i\), size of any OBDD representing \(\mathcal{F}_{\varphi,
         G_i}\) is at least \(f(w)\cdot n^{w/4}\) where \(n\) is the number of
         vertices of \(G_i\).
   \end{itemize}
\end{theorem}

We prove theorems~\ref{thm:sdd-ub} and~\ref{thm:obdd-ub} by describing an
algorithm that constructs the respective decision diagrams. The construction is
based on a decision procedure for checking satisfiability of the given \MSO{}
formula \(\varphi\) on a given graph \(G=(V, E)\) with a given nice tree
decomposition \(T\). The decision procedure is described in
Section~\ref{sec:decproc}. It is very similar to the approach used
by~\citet{COURCELLE199012}. It is based on dynamic programming based on the
structure of the tree decomposition. The procedure manages states associated
with each node of the decomposition. The state space depends on the structure of
formula \(\varphi\) --- we need a different set of states for atomic formulas
such as equality or adjacency, then the state spaces for atomic formulas can be
combined into the state spaces for more complex formulas. The key feature of the
decision procedure is that the number of the states in the state space does not
depend on graph \(G\), but only on the structure of \(\varphi\) and the width
\(w\) of the decomposition \(T\). Most work is done in the forget and join
nodes. For join nodes, the procedure for assigning the state to a node of \(T\)
based on the states of its child nodes depends only on the states and not on any
further values of variables in \(\varphi\). For forget nodes, the procedure that
assigns a state to a decomposition node based on the state of the child node
depends only on the variables from \(\domPhiG\) that are associated with the
vertex and edges of \(G\) that are being forgotten. The set of the involved
decision variables is called the context of the forget node. Context is local to
the forget node, no variables from it are part of a context of another forget
node. This locality then allows us to describe the state transformation as a
decision diagram, specifically as an SDD in Section~\ref{sec:sdd-ub} or as an
OBDD in Section~\ref{sec:obdd-ub}.

The procedure needs to be extended to detect inconsistent assignments which is described in
Section~\ref{ssec:inconsistent-assign}. To this end, state spaces need to be
extended with bit vectors keeping the status of each variable
(assigned/unassigned).

Theorem~\ref{thm:obdd-lb} relies on a lower bound by~\citet{razgon2014obdds},
we show that the problem used for the separation can be specified with a
\MSO{} formula.

The structure of the proofs is as follows. We first describe the decision
procedure in Section~\ref{sec:decproc}, including the consistency checking part
in Section~\ref{ssec:inconsistent-assign}.
The subsequent sections prove the above theorems. In particular,
Theorem~\ref{thm:sdd-ub} is proved in Section~\ref{sec:sdd-ub},
Theorem~\ref{thm:obdd-ub} is proved in Section~\ref{sec:obdd-ub}, and
Theorem~\ref{thm:obdd-lb} is proved in Section~\ref{sec:obdd-lb}.


\section{Deciding \texorpdfstring{\MSO{}}{MSO2} By Dynamic Programming}\label{sec:decproc}

The purpose of this section to describe a dynamic algorithm for deciding if a
given \MSO{} formula is satisfiable for a given graph. The procedure is
motivated by other similar procedures and approaches to the Courcelle's
theorem~\cite{COURCELLE199012, LRRS12, henriksen1995mona}, however, we need to
describe the procedure in a way that is suitable for the subsequent step in
which we describe the construction of a decision diagram representing the models
of the formula.

The procedure is defined recursively based on the structure of the input \MSO{}
formula \(\varphi\). It is based on the construction of the state tables
associated with the nodes of the tree decomposition of the input graph \(G=(V,
E)\). These state tables will then be used as the basis for the construction of
an SDD in Section~\ref{sec:sdd-ub}.

We shall start by describing a specific vertex coloring in
Section~\ref{ssec:vertex-color} that will allow us to succinctly store node
index in a state. Then we will specify the context of the decision procedure in
Section~\ref{ssec:dec-context}. The decision procedure itself is described in
Section~\ref{sectionstate}. Finally, in Section~\ref{ssec:inconsistent-assign},
we will describe a way of dealing with assignments to the decision variables
that are not consistent (according to Definition~\ref{def:consistent}).

\subsection{Good Vertex Coloring}\label{ssec:vertex-color}

When working with adjacency predicate in the main algorithm, we need to be able
to remember a specific vertex of the graph. To keep the complexity of the state
space in the algorithm bounded by a parameterized constant, we will represent
the vertex using a number from \(1\) to $(w+1)$ (we will use \([w+1]\) to denote
the set of these numbers), where $w$ is the width of the tree decomposition used
in the decision procedure. For this purpose, we will define \textbf{good vertex
coloring}.

\begin{definition}[Good vertex coloring]\label{def:color}
   Let $G=(V, E)$ be an undirected graph with nice tree decomposition
   $T=(V_T, E_T, \ell, r)$ of width $w$. We say that a function
   \(c: V\to [w+1]\) is a \textbf{good vertex coloring of $G$ with respect to
   $T$} if for every node \(p\in V_T\) and every pair of vertices \(u,
   v\in \ell(p)\) we have that \(c(u)\neq c(v)\).
\end{definition}

Let us note that a good vertex coloring \(c\) of \(G\) is a graph coloring of
\(G\) in the usual sense, because if \(\{u, v\}\in E\) is
an edge, then by condition~\ref{t2} we have that \(\{u, v\}\subseteq\ell(p)\)
for some \(p\in V_T\) and thus \(c(u)\neq c(v)\).

We will now show that such colorings exist for all graphs and decompositions.

\begin{lemma}[Existence of a good vertex coloring]
Let $G=(V, E)$ be an undirected graph with nice tree decomposition
$T=(V_T, E_T, \ell, r)$ of width $w$. Then there exists a function $c\colon V
\to [w+1]$ which is a good vertex coloring of $G$ with
respect to $T$.
\end{lemma}
\begin{proof}
   For every node \(v\in V\), let \(t_v\) denote the \enquote{forget \(v\)} node
   which which is unique by Lemma~\ref{lemma:oneforget}. In addition, let
   \(d_v\) denote the depth of \(t_v\) in \(T\). We shall define \(c(v)\)
   recursively based on \(d_v\).

   \begin{itemize}
      \item If \(d_v=0\), then \(t_v\) is the root \(r\) of \(T\). We set
         \(c(v)\) to be an arbitrary value from \([w+1]\).
      \item If \(d_v>0\), we assume that all nodes \(u\in V\) with \(d_u<d_v\)
         have \(c(u)\) already assigned. This applies to all nodes
         in \(\ell(t_v)\), because they must be forgotten on the path from
         \(t_v\) to the root. We have that \(|\ell(t_v)|\leq w\),
         because the child \(s\) of \(t_v\) has label \(\ell(s)=\ell(t_v)\cup\{v\}\),
         and \(|\ell(s)|=|\ell(t_v)|+1\leq w+1\). The set of colors
         \(W=[w+1]\setminus\{c(u)\mid u\in\ell(t_v)\}\) is thus nonempty and we
         can set \(c(v)\) to be any value from \(W\).
   \end{itemize}

   Let us now show that \(c\) is a good coloring. Let us consider a node \(p\in
   V_T\) and vertices \(u, v\in\ell(p)\). Both forget nodes \(t_u\) and \(t_v\)
   lie on the path from \(p\) to the root \(r\). A single node from \(V_T\) may
   only forget a single vertex from \(V\), \(t_u\) and \(t_v\) are thus two
   different nodes. Assume without loss of generality that \(d_u<d_v\). It
   follows that \(u\in\ell(t_v)\) and thus by the definition of \(c(v)\) we have
   that \(c(v)\neq c(u)\).
\end{proof}

\subsection{Decision Procedure Context}\label{ssec:dec-context}

Most work of the decision procedure is done when processing forget nodes.
Only decision variables related to the vertex and edges being forgotten are
involved in this process. These decision variables together form the
context associated with the forget node.

\begin{definition}[Forget node context]
   Let \(\varphi\) be an \MSO{} formula and
   let $G=(V, E)$ be an undirected graph with nice tree decomposition
   $T=(V_T, E_T, \ell, r)$ of width $w$. Let $c$ be a good vertex coloring of $G$ with
   respect to $T$. Let $p$ be a \enquote{forget $u$} node in $T$ for some \(u\in
   V\), in addition, let \(E_F\subseteq E\) denote the set of edges that are
   being forgotten in \(p\).
   We define the \textbf{context} of node $p$, denoted
   $\CTX_{\varphi, T}(p)$, to be the following subset of $\mathcal{D}_{\varphi, G}$:
   \begin{align*}
          &\{[x = u] \mid x \text{ is a free vertex object variable in $\varphi$}\}\\
      \cup\; &\{[x = e] \mid x \text{ is a free edge object variable in $\varphi$ and $e\in E_F$}\}\\
      \cup\; &\{[u \in X] \mid X \text{ is a free vertex set variable in $\varphi$}\}\\
      \cup\; &\{[e \in X] \mid X \text{ is a free edge set variable in $\varphi$ and $e\in E_F$}\}\\
   \end{align*}
\end{definition}

Lemmas~\ref{lemma:oneforget} and~\ref{lemma:oneforgete} imply that every decision
variable belongs to exactly one context. Restricting the decision procedure to
access only the data stored in the context makes it possible to use the
procedure as the basis for the construction of a decision diagram.

\begin{lemma}[Size of a context]%
\label{lemma:soc}
For a tree decomposition $T$ of width $w$, any forget node $p$ in $T$, and
formula $\varphi$, the size of $\CTX_{\varphi, T}(p)$ is at most
$|\varphi|\cdot(w+1)$
\end{lemma}
\begin{proof}
   Assume that $p$ is the \enquote{forget $u$} node and its only child is $q$.
   The node $q$ may contain up to $w+1$ vertices in its label, so the vertex $u$
   may be connected to at most $w$ other nodes from $\ell(q)$. Therefore, we are
   forgetting at most $w$ edges. Since we are also forgetting the vertex $u$,
   there are up to $w+1$ graph objects that we are forgetting in total.

   For each graph object, there are as many other decision variables as there
   are free variables in $\varphi$, we thus have at most $|\varphi|$ decision
   variables per forgotten graph object. In total, we can use $(w+1)|\varphi|$
   as an upper bound for the size of $\CTX_{\varphi, T}(p)$.
\end{proof}

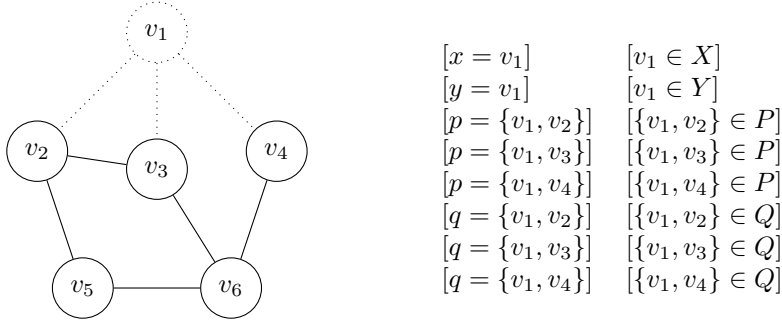
\begin{figure}
\begin{center}

\begin{tikzpicture}[vertex/.style={circle, draw, minimum size=8mm}]
   \node (v1) [vertex, dotted] {\(v_1\)};
   \node (v2) [vertex, below left=of v1] {\(v_2\)};
   \node (v3) [vertex, below=of v1] {\(v_3\)};
   \node (v4) [vertex, below right=of v1] {\(v_4\)};
   \node (v5) [vertex, below left=1cm and 4mm of v3] {\(v_5\)};
   \node (v6) [vertex, below right=1cm and 4mm of v3] {\(v_6\)};

   \draw [dotted] (v1) -- (v2);
   \draw [dotted] (v1) -- (v3);
   \draw [dotted] (v1) -- (v4);
   \draw (v2) -- (v3);
   \draw (v2) -- (v5);
   \draw (v3) -- (v6);
   \draw (v4) -- (v6);
   \draw (v5) -- (v6);

   \node [right=3cm of v1, anchor=north west] {
      \begin{tabular}{l l}
         \([x=v_1]\) & \([v_1\in X]\)\\
         \([y=v_1]\) & \([v_1\in Y]\)\\
         \([p=\{v_1, v_2\}]\) & \([\{v_1, v_2\}\in P]\)\\
         \([p=\{v_1, v_3\}]\) & \([\{v_1, v_3\}\in P]\)\\
         \([p=\{v_1, v_4\}]\) & \([\{v_1, v_4\}\in P]\)\\
         \([q=\{v_1, v_2\}]\) & \([\{v_1, v_2\}\in Q]\)\\
         \([q=\{v_1, v_3\}]\) & \([\{v_1, v_3\}\in Q]\)\\
         \([q=\{v_1, v_4\}]\) & \([\{v_1, v_4\}\in Q]\)\\
      \end{tabular}
   };
\end{tikzpicture}
\end{center}
\caption{An example of a forget node context when forgetting node \(v_1\) and
all of its incident edges.}%
\label{fig:ctx}
\end{figure}

Figure~\ref{fig:ctx} shows an example of a forget node context. The vertex that
is being forgotten has dotted outline. The edges that are being forgotten also
are dotted. The decision variables from the context are listed. The context was
constructed for a formula, with free vertex object variables $x$ and $y$, free
edge object variables $p$ and $q$,  free vertex set variables $X$ and $Y$ and
free edge set variables $P$ and $Q$.

\subsection{Decision Procedure}%
\label{sectionstate}

We are ready to describe the decision procedure which determines the value of a
given \MSO{} formula \(\varphi\) on a given graph \(G=(V, E)\) and a
\(\varphi\)-assignment \(\alpha\). It works with the decision variables
introduced in Definition~\ref{def:decision-variables} rather
than directly working with the object and set variables from \(\varphi\). This
will allow us to then turn the description of the decision procedure into a
construction of a decision diagram representing the models of \(\varphi\).

The decision procedure is based on a dynamic programming which follows the
structure of a nice tree decomposition \(T=(V_T, E_T, \ell, r)\) of \(G\).
During its work, the procedure assigns a state to each node of the tree
decomposition. For each type of node, we will describe a lookup table which
allow us to determine the new state depending on the states of the child nodes.
The state space depends on the structure of the formula. Most of the work is
done in forget nodes where the lookup table has to consider the values of
context variables.

The procedure starts in the leaves with a fixed initial state, the state is then
updated in each forget node. In addition, the state will be updated in a join
node where two states coming from the child nodes need to be combined. We will
define the possible states by induction on the structure of \(\varphi\). When
the procedure gets to the root, we only check if the root state is one of the
accepting states.

Since we only rely on $\varphi$ and the number of different colors bounded by
$w+1$ to design the state spaces and tables, their size will be a function of
the parameter \(k=|\varphi|+w\) which is a constant for fixed $w$ and
$|\varphi|$. We will first assume that assignments to decision variables given
throughout the procedure are consistent (see Definition~\ref{def:consistent}).
In Section~\ref{ssec:inconsistent-assign} we will then describe how checking of
consistency of the assignment can be incorporated into the decision procedure.

We will prove correctness of the decision procedure. The decision procedure is
correct for $\varphi$ and $G$, if it evaluates to accepting state if and only if
$\varphi$ is true for a given consistent $\varphi$-assignment. 

\subsubsection{Overview Of The State Spaces}

When it comes to \MSO{} formulas, the building blocks can be split into \(3\) categories:

\begin{itemize}
\item Atomic formulas
\item Boolean connectives
\item Existential quantification
\end{itemize}

The state space of equality and membership will consist of states $\SI$ and
$\ST$, where only $\ST$ is accepting. For false instantiations of these atomic
formulas, the computation will stay in the state $\SI$. The state space of an adjacency predicate
will consist of states $\SI$, $\ST$, and an additional
state for each color. To see why we need more states, consider an edge $e$ with
endpoints $u$ and $v$. If we would like to determine $\Adj(x, y)$ with
$\alpha(x)=v$ and $\alpha(y)=e$ while forgetting vertex $u$ and edge $e$, we
will need to delay the decision until we are forgetting $v$, since at the time
of forgetting $u$, the decision procedure has no information on whether
$\alpha(y)$ is equal to $v$, because it has access only to the context of the
procedure. To delay the decision until $v$ is being forgotten,
we store the color of $v$ in the state.

For negation, we will swap the role of accepting states and non-accepting
states. For conjunction, we will use product state space, where a product state
is accepting if and only if both of the constituent states were accepting for
both conjuncts.

The state space of existential quantification of the form $(\exists x_1, \dots,
x_k)\psi$ in which several variables that were free in \(\psi\) are now bound,
we will use subsets of states from the state space of $\psi$. These subsets will
represent all reachable states in the decision procedure for $\psi$ for
different instantiations of the bound variables. Additional bits of information
will be added to the states from the state space for $\psi$ to check consistency
of assignments to the bound variables.

Considering multiple bound variables at once will allow us to prove that the power tower
in the estimate on the size of the state space for a formula in prenex form
depends on the number of quantifier alterations, rather than the length of the
formula.

\subsubsection{Definition, Notations, And The Full Procedure}

For the description of the procedure, let us fix an undirected graph \(G=(V,
E)\) together with its nice tree decomposition \(T=(V_T, E_T, \ell, r)\) of
width \(w\). Let us also fix an \MSO{} formula \(\varphi\) which we wish to
evaluate and a good vertex coloring \(c\) of graph \(G\) with respect to the
tree decomposition \(T\). With \(\varphi\), we shall associate its state space
\(S_\varphi\), initial state \(s_\varphi\), and the set of accepting states
\(A_\varphi\).  We shall describe a procedure that evaluates \(\varphi\) on a
given assignment \(\alpha\). The procedure will use two lookup tables for
\(\varphi\). Forget lookup table \(\FORGET_{\varphi, w}(s, \delta)\) is defined
for every forget node \(p\). It returns
the state of \(p\) given the state \(s\) of the child node of \(p\)
and the \(\mathcal{D}_{\varphi, G}\) assignment \(\delta\) representing
\(\varphi\)-assignment \(\alpha\) restricted to the decision variables in the
context \(\CTX_{\varphi, T}(p)\). Join lookup table \(\JOIN_{\varphi, w}(s_L,
s_R)\) is associated with every join state \(p\).
It returns the state of \(p\) given the states \(s_L\) and
\(s_R\) of the left and right child of \(p\).

Before we start describing state spaces and lookup tables for the formula
\(\varphi\) depending on its structure, we will describe the full decision
procedure which assumes the existence of the above building blocks.
The full procedure is given in Algorithm~\ref{alg:full} and it uses function
\(\nodedp\) defined by Algorithm~\ref{alg:node} as a subroutine for individual
decomposition nodes.

\begin{algorithm}
\caption{\(\nodedp(\alpha, p)\)}\label{alg:node}
\begin{algorithmic}
   \Require{} $\alpha$, the assignment to free variables of $\varphi$
   \Require{} $p$, current decomposition node.
   \Ensure{} A state from $S_\varphi$
   \If{$p$ is a leaf node}
      \State\Return{$s_\varphi$}
   \ElsIf{$p$ is an introduce node}
      \State{} $q \gets$ the only child of $p$
      \State\Return{$\nodedp(\alpha, q)$}
   \ElsIf{$p$ is a forget node}
      \State{} $q \gets$ the only child of $p$
      \State{} $s_q \gets \nodedp(\alpha, q)$
      \State{} \(\delta = \mathcal{D}_{\varphi, G}\)-assignment representing \(\alpha\) restricted to $\CTX_{\varphi, T}(p)$
      \State{} \Return{$\FORGET_{\varphi, w}(s_q, \delta)$}
   \ElsIf{$p$ is a join node}
      \State{} $q_R, q_L \gets$ the two children of $p$
      \State{} $s_L \gets Node\_decision\_procedure(\alpha, q_L)$
      \State{} $s_R \gets Node\_decision\_procedure(\alpha, q_R)$
      \State\Return{$\JOIN_{\varphi,w}(s_L, s_R)$}
   \EndIf%
\end{algorithmic}
\end{algorithm}

\begin{algorithm}
   \caption{Full decision procedure}\label{alg:full}
   \begin{algorithmic}
      \Require{} $\alpha$, the assignment to free variables of $\varphi$
      \Ensure{} Evaluation of $\varphi$ on $\alpha$
      \State{} $s \gets Node\_decision\_procedure(\alpha,r)$
      \State\Return{boolean value based on whether $s \in A_\varphi$}
   \end{algorithmic}
\end{algorithm}

The definition of the lookup tables and states for $\varphi$ will be considered
\textbf{correct}, if for any graph $G$, tree decomposition $T$, good coloring $c$ and
consistent $\mathcal{D}_{\varphi, G}$-assignment $\delta$ representing a
$\varphi$-assignment $\alpha$, the decision procedure
correctly evaluates $\varphi$ for assignment $\alpha$.

We will now design the state spaces for any $\varphi$ by
recursion on its structure and also prove the correctness of the designed
procedure. In addition, we also define the lookup tables for each forget and
join node in the decomposition \(T\).

\subsubsection{State Space For Equality}

Let $\varphi \equiv (x = y)$, where $x$ and $y$ are object variables of the same
sort. Regardless of the variable sort, the state space is defined as $S_\varphi := \{\SI,
\ST\}$, with initial state set to $s_\varphi := \SI$ and the set of accepting
states is defined as $A_\varphi := \{\ST\}$.

\paragraph{Forget lookup table}

Let $s$ be the input state and $\delta$ be the context assignment of a forget
node. Let $v_F$ be the forgotten vertex and let $E_F$ be the set of all
forgotten edges in the forget node.

The output state \(s'=\FORGET_{\varphi, w}(s, \delta)\) is defined as \(\ST\) if
\(s=\ST\). If \(s=\SI\), then the state changes to \(\ST\) if \(x\) and \(y\)
are both equal to the same forgotten object according to \(\delta\). More
formally, if \(x\) and \(y\) are vertex object variables, then we set

\begin{equation}\label{eq:forget-equal-vert}
   s'=\begin{cases}
      \ST & \text{if \(s=\ST\)}\\
      \ST & \text{if \(\delta([x = v_F]) \wedge \delta([y = v_F])\)}\\
      \SI & \text{otherwise.}
   \end{cases}
\end{equation}

If both \(x\) and \(y\) are edge object variables, then we set

\begin{equation}\label{eq:forget-equal-edge}
   s'=\begin{cases}
      \ST & \text{if \(s=\ST\)}\\
      \ST & \text{if \(\delta([x = e_F]) \wedge \delta([y = e_F])\) for some \(e_F\in E\)}\\
      \SI & \text{otherwise.}
   \end{cases}
\end{equation}

\paragraph{Join lookup table}

The join lookup table is defined below. Note that two states $\ST$ cannot
actually meet and the result can be set arbitrarily. The result is set to $\ST$ to
simplify the proof of correctness.

\begin{center}
   \begin{tabular}{ c | c c }
      $s_L, s_R$ & $\SI$ & $\ST$ \\ 
      \hline
      $\SI$ & $\SI$ & $\ST$ \\ 
      $\ST$ & $\ST$ & $\ST$ \\
   \end{tabular}
\end{center}

\paragraph{Correctness}
Suppose that $\varphi \equiv (x = y)$ is true for a given \(\varphi\)-assignment
$\alpha$. In other words, $\alpha(x)=\alpha(y)$. Let $z$ be the graph object
assigned to both $\alpha(x)$ and $\alpha(y)$.

By forget node uniqueness (Lemma \ref{lemma:oneforget} or
\ref{lemma:oneforgete} depending on the sort of \(x\) and \(y\)), $z$ is
forgotten in a specific decomposition node $p$. When
Algorithm~\ref{alg:node} is evaluated for $p$, the output state is set to
\(\ST\) by the second line of~\eqref{eq:forget-equal-vert}
or~\eqref{eq:forget-equal-edge} depending on the sort of \(z\). Note that the
input state must be \(\SI\) in this case, because that is the initial state
assigned to the leaves and \(p\) is the only node that forgets \(z\). Then, on
the path from \(p\) to the root \(r\), all states must be \(\ST\), because this
value is propagated up by~\eqref{eq:forget-equal-vert}
or~\eqref{eq:forget-equal-edge}, and the definition of the join lookup table.

Since the state $\ST$ is propagated to the root, it is considered the final
state. Since $\ST$ is an accepting state, the procedure correctly decides, that
\(\varphi\) evaluates to false on $\alpha$.

Suppose now that $\varphi \equiv (x = y)$ is false for a given consistent
assignment $\alpha$. In other words, $\alpha(x)\ne\alpha(y)$.

We will show by induction, that the state assigned to the root is $\SI$.

Leaf nodes are assigned the initial state $\SI$. Introduce
nodes are assigned $\SI$ based on their child node, which is assigned $\SI$ by
induction hypothesis. A join node is assigned $\SI$, because both of its child
nodes have been assigned $\SI$ by induction hypothesis.

Finally consider a forget node. We assume that the input state is $\SI$, so
the first rule of~\eqref{eq:forget-equal-vert} or~\eqref{eq:forget-equal-edge}
does not apply. We also assume that $\alpha(x)\ne\alpha(y)$, so $\alpha(x)$ and
$\alpha(y)$ cannot be equal to the same graph object. This means that the third
rule of~\eqref{eq:forget-equal-vert} or~\eqref{eq:forget-equal-edge} applies and
the forget node is assigned $\SI$.

By induction, the root is assigned $\SI$, which is not an accepting state and
the procedure correctly decides, that \(\varphi\) evaluates to false on
$\alpha$.

\subsubsection{State Space For Membership}

Let $\varphi \equiv (x \in Y)$, where $x$ and $Y$ are object variable and set
variable respectively, both of the same sort.  Regardless of the variable sort,
the state space is defined as $S_\varphi := \{\SI, \ST\}$, with initial state set to
$s_\varphi := \SI$ and the set of accepting states is defined $A_\varphi := \{\ST\}$.

\paragraph{Forget lookup table}

Let $s$ be the input state and $\delta$ be the context assignment of a forget
node. Let $v_F$ be the forgotten vertex and let $E_F$ be the set of all
forgotten edges in the forget node.

The output state \(s'=\FORGET_{\varphi, w}(s, \delta)\) is defined as \(\ST\) if
\(s=\ST\). If \(s=\SI\), then the state changes to \(\ST\) if
the forgotten object is both equal to \(x\) and belongs to \(Y\) according to
\(\delta\).
More formally, if \(x\) is a vertex object variable and \(Y\) is a vertex set
variable, then we set

\begin{equation}\label{eq:forget-in-vert}
   s'=\begin{cases}
      \ST & \text{if \(s=\ST\)}\\
      \ST & \text{if \(\delta([x = v_F]) \wedge \delta([v_F\in Y])\)}\\
      \SI & \text{otherwise.}
   \end{cases}
\end{equation}

If \(x\) is an edge object variable and \(Y\) is an edge set
variable, then we set

\begin{equation}\label{eq:forget-in-edge}
   s'=\begin{cases}
      \ST & \text{if \(s=\ST\)}\\
      \ST & \text{if \(\delta([x = e_F]) \wedge \delta([e_F\in Y])\) for some \(e_F\in E\)}\\
      \SI & \text{otherwise.}
   \end{cases}
\end{equation}

\paragraph{Join lookup table}

The join lookup table is defined below. Note that two states $\ST$ cannot
actually meet and the result can be set arbitrarily. The result is set $\ST$ to
simplify the proof of correctness.

\begin{center}
   \begin{tabular}{ c | c c c }
      $s_L, s_R$ & $\SI$ & $\ST$ \\ 
      \hline
      $\SI$ & $\SI$ & $\ST$ \\ 
      $\ST$ & $\ST$ & $\ST$ \\  
   \end{tabular}
\end{center}

\paragraph{Correctness}
Suppose that $\varphi \equiv (x \in Y)$ is true for a given \(\varphi\)-assignment
$\alpha$. In other words, $\alpha(x) \in \alpha(Y)$.

By forget node uniqueness (Lemma \ref{lemma:oneforget} or \ref{lemma:oneforgete}
depending on the sort of \(x\) and \(Y\)), $\alpha(x)$ is forgotten in a
specific decomposition node $p$. When Algorithm~\ref{alg:node} is evaluated for $p$,
the output state is set to \(\ST\) using the second rule
of~\ref{eq:forget-in-vert} or~\ref{eq:forget-in-edge} depending on the sort of
\(x\) and \(Y\).

The state $\ST$ then propagates to the root
using~\eqref{eq:forget-in-vert},~\eqref{eq:forget-in-edge}, and the definition
of the join lookup table. The decision procedure thus correctly decides, that
\(\varphi\) evaluates to false on $\alpha$.

Suppose that $\varphi \equiv (x \in Y)$ is false for given consistent assignment
$\alpha$. In other words, $\alpha(x)\notin\alpha(Y)$.

We can show by induction that all nodes are assigned $\SI$ similarly as we did
for equality. The only difference is that we use the assumption
$\alpha(x)\notin\alpha(Y)$ to argue that the state cannot change from \(\SI\) to
\(\ST\) by applying the second rule of~\eqref{eq:forget-in-vert}
or~\eqref{eq:forget-in-edge}.

Since all nodes are assigned $\SI$, the root is assigned $\SI$, too, and the
decision procedure correctly decides, that
\(\varphi\) evaluates to false on $\alpha$.

\subsubsection{State Space For Adjacency}\label{sssec:adj}

Deciding adjacency is not as easy as deciding the previous two atomic formulas.
To illustrate, suppose that $v$ and $e$ are adjacent. Consider formula
$\Adj(x,y)$ and assignment $\alpha$ with $\alpha(x)=v$ and $\alpha(y)=e$, on
which the formula evaluates to true. The edge $e$ will be forgotten together
with one of its endpoints. The endpoint could be $v$ and then we can decide the
formula immediately. If we are forgetting the other endpoint (say \(u\)) first,
the decision procedure does not have variable \([x=v]\) in the context, and it
cannot thus check the validity of \(\Adj(x, y)\). It can do so when forgetting
\(v\), but at that moment, the procedure does not have variable \([y=e]\) in the
context. To pass the information from the \enquote{Forget \(u\)} node which also
forgets \(e\) to the \enquote{Forget \(v\)} node, we use color of \(v\) as the
state. Then in \enquote{Forget \(v\)} we check if the color in the state equals
the color of \(v\). This way, the size of the state space remains bounded by the
treewidth.

Let $\varphi \equiv \Adj(x, y)$, where $x$ is a vertex object variable and $y$ is
an edge object variable. The state space is defined as
$S_\varphi := \{\SI, \ST\} \cup
\{c_1, \dots, c_{w+1}\}$, where $w$ is the width of the decomposition \(T\).
The initial state set to
$s_\varphi := \SI$ and the set of accepting states is set to $A_\varphi := \{\ST\}$.

\paragraph{Forget lookup table}

Let $s$ be the input state and $\delta$ be the context assignment of a forget
node \(p\). Let $v_F$ be the forgotten vertex and let $E_F$ be the set of all
forgotten edges in the forget node. Recall that we assume a good coloring
\(c:V\to[w+1]\).

The output state \(s'=\FORGET_{\varphi,
w}(s, \delta)\) is obtained by the first applicable rule from the following
list.

\begin{enumerate}
   \item\label{enum:adj:1} If $s = \ST$, then \(s'=\ST\).
   \item\label{enum:adj:2} If \(s=c_i\) for some \(i\in[w+1]\) and \(c(v_F)=i\), then \(s'=\ST\) if $\delta([x = v_F])=1$ and \(s'=\SI\) otherwise.
   \item\label{enum:adj:3} If \(s=c_i\) for some \(i\in[w+1]\), then \(s'=s\).
   \item\label{enum:adj:4} If \(\delta([y=e])=1\) for some \(e\in E_F\) and \(\delta([x=v_F])=1\),
      then \(s'=\ST\).
   \item\label{enum:adj:5} If \(\delta([y=e])=1\) for some \(e=\{u, v_F\}\in E_F\), then \(s'=c_i\)
      where \(i=c(u)\).
   \item\label{enum:adj:6} Otherwise \(s'=\SI\).
\end{enumerate}

Note that
rule~\eqref{enum:adj:3} is only applicable if rule~\eqref{enum:adj:2} was not
and therefore \(c(v_F)\neq i\) in this case. Similarly, \(s=\SI\) in
rules~\eqref{enum:adj:4}-\eqref{enum:adj:6}.

\paragraph{Join lookup table}

The join lookup table is defined below. The states denoted as \enquote{cannot
occur} represent the situations that cannot happen during the work of
Algorithm~\ref{alg:full} and can thus be set arbitrarily without affecting the
correctness.

\begin{center}
\begin{tabular}{ c | c c c c }
 $s_L, s_R$ & $\SI$ & $\ST$ & $c_i$\\ 
\hline
 $\SI$ & $\SI$ & $\ST$ & $c_i$ \\ 
 $\ST$ & $\ST$ & Cannot occur & Cannot occur \\
$c_j$ & $c_j$ & Cannot occur & Cannot occur \\
\end{tabular}
\end{center}

\paragraph{Correctness}

Recall that \(\alpha\) denotes the assignment of values to the free variables of
\(\varphi\equiv\Adj(x, y)\). Let \(e_y=\alpha(y)\) be the edge assigned to
variable \(y\) and assume that \(e_y=\{u, v\}\) for vertices \(u, v\in V\).
Let
\(p_0\) denote the forget node of \(T\) which forgets \(e_y\) and let \(P=(p_0,
\dots, p_t)\) denote the path from \(p_0\) to the root \(r=p_t\) of \(T\). Let
us start with the following claim.

\begin{claim}\label{claim:adj}
   Assume \(q\in V_T\) is a node that is not on path \(P\), then the
   state assigned to \(q\) is \(\SI\)
\end{claim}

\begin{proof}
   We shall proceed by induction on the height \(h(q)\) of \(q\), i.e., the maximum length
   of a path from \(q\) to the reachable leaves. If \(h(q)=0\), then \(q\) is a
   leaf and it is thus assigned the initial state \(\SI\). Assume now that
   \(h(q)>0\) and that \(q\) is not on path \(P\). If \(q\) is an introduce
   node, then by Algorithm~\ref{alg:node}, the state of \(q\) is equal to the
   state of its only child node which is \(\SI\) by induction hypothesis. If
   \(q\) is a join node with child nodes \(q^L\) and \(q^R\), then their states
   are \(\SI\) by induction hypothesis. Using the join lookup table we have that
   \(q\) is \(\SI\) as well.

   Consider now that \(q\) is a forget node, then the state \(s'\) of \(q\) is
   determined by the above rules~\eqref{enum:adj:1}-\eqref{enum:adj:6}. The
   input state is \(\SI\) by induction hypothesis and thus
   rules~\eqref{enum:adj:1}-\eqref{enum:adj:3} are not applicable. Since \(q\)
   is not on \(P\), it does not forget edge \(e_y\) and thus \(e_y\not\in E_F\).
   Since \(\alpha(y)=e_y\), we have
   \(\delta([y=e_y])=1\), but because \(e_y\not\in E_F\) and \(\delta\) is
   consistent, we have that \(\delta([y=e])=0\) for each \(e\in E_F\).
   Rules~\eqref{enum:adj:4} and~\eqref{enum:adj:5} are not applicable either.
   The only applicable rules is thus~\eqref{enum:adj:6} which sets \(s'=\SI\).
\end{proof}

Note that by Claim~\ref{claim:adj}, at least one of the input nodes of every
join node must have state \(\SI\) and the four cells in the join lookup table
labeled as \enquote{Cannot occur} are labeled correctly.

Let us now assume that \(\alpha\) satisfies \(\varphi\) and let us show that the
root gets assigned the accepting state \(\ST\).  In particular, we shall assume
that \(\alpha(x)=u\).  We need to consider two cases depending on which vertex of
\(e_y\) is forgotten first.

\begin{itemize}
   \item \emph{Node \(p_0\) is \enquote{forget \(u\)}.} When determining state
      \(s'\) of \(p_0\), the input state is \(s=\SI\),
      rules~\eqref{enum:adj:1}-\eqref{enum:adj:3} are thus not applicable. We have
      that \(e_y\in E_F\) and \(\delta([y=e_y])=1\). In addition, \(\alpha(x)=u\) and
      \(v_F=u\) is the forgotten vertex. It follows that \(\delta([x=v_F])=1\).
      Rule~\eqref{enum:adj:4} thus sets \(s'=\ST\) which is then propagated along path
      \(P\) to the root using rule~\eqref{enum:adj:1} for the forget nodes and the join
      lookup table for the join nodes.
   \item \emph{Node \(p_0\) is \enquote{forget \(v\)}.} Because \(p_0\) forgets
      \(e_y\), \(u\in\ell(p_0)\) and thus it is forgotten somewhere along path \(P\)
      in node \(p_a\) for some \(a\in\{1, \dots, t\}\). When processing \(p_0\), the
      input state is \(\SI\) and the state is set by rule~\eqref{enum:adj:5} to
      \(c_i\) where \(i=c(u)\). This state is propagated up to the node \(p_a\) by
      rule~\eqref{enum:adj:3} in forget nodes and by the join lookup table in the join
      nodes. When processing \(p_a\), the state is changed to \(\ST\) using
      rule~\eqref{enum:adj:2}, because at that time \(v_F=u\), \(c(v_F)=c(u)=i\), and
      \(\delta([x=u])=\delta([x=v_F])=1\) due to having \(\alpha(x)=u\). State \(\ST\)
      is then propagated to the root using rule~\eqref{enum:adj:1} for the forget
      nodes and join lookup table for the join nodes.
\end{itemize}

Let us now assume that \(\varphi\) is not satisfied by \(\alpha\), thus
\(\alpha(x)\not\in\{u, v\}\). We shall
prove that in this case, no node of the tree decomposition \(T\) has state
\(\ST\). By Claim~\ref{claim:adj}, it is enough to show that this holds for the
path \(P\). Assume without loss of generality that \(p_0\) forgets \(v\) and
\(u\) is forgotten in node \(p_a\) for some \(a\in\{1, \dots, t\}\). Assume that
\(i=c(u)\) is the color of \(u\). We claim
that nodes \(p_0, \dots, p_{a-1}\) get assigned state \(c_i\) and nodes \(p_a,
\dots, p_t\) get assigned state \(\SI\).

In both cases, we shall proceed by induction.

\begin{itemize}
   \item When processing \(p_0\), the input state is \(\SI\) by
      Claim~\ref{claim:adj}, rules~\eqref{enum:adj:1}-\eqref{enum:adj:3} are
      thus not applicable. Rule~\eqref{enum:adj:4} is not applicable, because
      \(v_F=v\), \(\alpha(x)\neq v\), and thus \(\delta([x=v_F])=0\).
      Rule~\eqref{enum:adj:5} is therefore applied with \(e=e_y\) and \(v_F=v\).
      State \(s'=c_i\), since \(i=c(u)\).
   \item Consider now a node \(p_b\) for some \(0<b<a\). By induction hypothesis, the
      input state \(s\) is \(c_i\). Denote \(s'\) the state associated with
      \(p_b\). If \(p_b\) is an introduce node, then state \(s'=s=c_i\) by
      Algorithm~\ref{alg:node}. If \(p_b\) is a join node, then the \(p_{b-1}\)
      is one of the childs of \(p_b\) and the other child has state \(\SI\) by
      Claim~\ref{claim:adj}. By join lookup table, \(s'=s=c_i\). If \(p_b\) is a
      forget node, then \(v_F\not\in\{u, v\}\) by assumption on \(b\).
      We have \(c(v_F)\neq
      c(u)\) by the definition of good vertex coloring
      (Definition~\ref{def:color}), because \(u\) has not yet been forgotten and
      must be present in \(\ell(p_j)\). It follows that rule~\eqref{enum:adj:2}
      is not applicable. Since \(s=c_i\), the first applicable rule is
      thus~\eqref{enum:adj:3} which sets \(s'=s=c_i\).
   \item When processing \(p_a\), the input state \(s\) is \(c_i\) by induction
      hypothesis and \(v_F=u\), thus \(c(v_F)=i\). It follows that
      rule~\eqref{enum:adj:2} is applied. Since \(\alpha(x)\neq u\), we have
      that \(\delta([x=v_F])=0\) and \(s'=\SI\).
   \item Assume now node \(p_b\) for \(b>a\). If \(p_b\) is an introduce node,
      then it gets assigned \(\SI\) by induction hypothesis. If \(p_b\) is a
      join node, then both child nodes of \(p_b\) must have state \(\SI\) (one
      because it is not on the path \(P\) and the other by induction
      hypothesis). If \(p_b\) is a forget node, then
      rules~\eqref{enum:adj:1}-\eqref{enum:adj:3} are not applicable,
      because \(s=\SI\) by induction hypothesis. Rules~\eqref{enum:adj:4}
      and~\eqref{enum:adj:5} are not applicable, because \(\alpha(y)=e_y\) and
      \(e_y\not\in E_F\) and thus \(\delta([y=e])=0\) for each \(e\in E_F\) by
      consistency of \(\delta\). Thus the only applicable rule
      is~\eqref{enum:adj:6} which sets \(s'=s=\SI\).
\end{itemize}

\subsubsection{State Space For Negation}%
\label{sectionneg}

Let $\varphi \equiv \neg \psi$, where $\psi$ is an \MSO{} formula. The state space
is $S_\varphi := S_\psi$, the initial state is $s_\varphi := s_\psi$ and the set
of accepting states is $A_\varphi := S_\psi \setminus A_\psi$. The state tables
remain the same as for $\psi$.

\paragraph{Correctness}
Assume that the full decision procedure for $\psi$ is correct. The only thing
that has changed from the procedure for $\psi$ to the procedure for $\varphi$ is
the set of accepting states, which is used only at the very end of
Algorithm~\ref{alg:full}. Having \(A_\varphi=S_\psi\setminus A_\psi\) implies
that the answer of the algorithm is the negation of what would be the answer for
\(\psi\).

\subsubsection{State Space For Conjunction}%
\label{sectionconj}

Let $\varphi \equiv \psi \wedge \rho$, where $\psi$ and $\rho$ are \MSO{} formulas.

The state space is $S_\varphi := S_\psi \times S_\rho$, the initial state is
$s_\varphi := (s_\psi, s_\rho)$ and the set of accepting states is $A_\varphi :=
A_\psi \times A_\rho$. The lookup tables are defined as:
\begin{align*}
   \FORGET_{\varphi,w}((s_1,s_2),\delta) &:= (\FORGET_{\psi,w}(s_1,\delta),\FORGET_{\rho,w}(s_2,\delta))\\
   \JOIN_{\varphi,w}((s_{L1},s_{L2}),(s_{R1},s_{R2})) &:= (\JOIN_{\psi,w}(s_{L1},s_{R1}),\JOIN_{\rho,w}(s_{L2},s_{R2}))
\end{align*}

\paragraph{Correctness}
Assume that the decision procedures for $\psi$ and $\rho$ are correct. Assume
that the final state for a given instance of the problem and formula $\psi$ is
$f_1$, while the final state for $\rho$ is $f_2$. If we project the state space
in the computation for $\varphi$ to just the first item of the pair, we will see
that it is the same as for $\psi$. Likewise, if we project to the second item in
the pair, the computation will be the same as for $\rho$. The final state when
computing for $\varphi$ is therefore $(f_1,f_2)$. By doing the Cartesian product
of the sets, this is accepting state if and only if $f_1$ is accepting for
$\psi$ and $f_2$ is accepting for $\rho$, which correctly models conjunction of
the subformulas.

\subsubsection{State Space For Quantification}%
\label{sectionstateex}

In this subsection, we will consider general existential quantification of the
following form:
\begin{equation}\label{eq:quant-form}
   \begin{aligned}
      \varphi \equiv (&\exists x_1,\dots,x_{\kvo} \in V,\\
                      &\exists y_1,\dots,y_{\keo} \in E,\\
                      &\exists X_1,\dots,X_{\kvs} \subseteq V,\\
                      &\exists Y_1,\dots,Y_{\kes} \subseteq E\\
      )& [\psi]
   \end{aligned}
\end{equation}

This general form includes all of the four elementary constructions in the
language of \MSO{}. We consider this general form instead of a form in which a
single variable would be quantified so that we can show a bound on the size of
the state space based on the number of quantifier alternations.

When evaluating formula \(\varphi\) on an assignment \(\alpha\), we are also
evaluating \(\psi\). Assignment \(\alpha\) assigns values only to the variables
that are free in \(\varphi\). In addition to them, all variables bound
by~\eqref{eq:quant-form} are also free in \(\psi\). Before using the procedure
for evaluating \(\psi\), we must thus extend the assignment by assigning the
values of these newly bound variables and we must consider each such extension
\(\alpha'\). When processing a decomposition node \(t\), we need to keep track
of the assignments performed in the nodes below \(t\). In particular, consider a
vertex object variable \(x_i\). When processing a forget node \(t\) which
forgets variable \(v\), we need to consider two possibilities in which
\(\alpha'(x_i)\) is either set to \(v\), or not. The former is only possible if
\(x_i\) has not yet been assigned any value in the subtree of \(t\). In this
case \(\alpha'(x_i)\) is a node that has already been forgotten. To keep track
of these assignments, we use a bit vector of length \(\kvo\). In fact, we need to
do the same for edge object variables as well and thus the bit vector has length
\(\kvo+\keo\).

An extension \(\alpha'\) to \(\alpha\) is thus associated with a \textbf{state
pair} \((s, b)\) where \(s\in S_\psi\) is a state assigned to the current node
\(p\) when processing \(\psi\) and \(b\) is a bit vector of length \(\kvo+\keo\)
which specifies for each bound object variable \(x\) whether \(\alpha'(x)\) has
been forgotten in the subtree of \(p\) or not. Denote \(\mathcal{S}_\varphi\)
the set of all such state pairs. We then define the state space
\(S_\varphi:=\mathcal{P}(\mathcal{S}_\varphi)\), i.e., each state is a set of
state pairs defined above.

To not confuse states from $S_\varphi$ and $S_\psi$, we will use the term
\textbf{state set} to refer to states from $S_\varphi$, as they are sets of
states from $S_\psi$ with additional bits of information.

The initial state set is defined as $s_\varphi = \{(s_\psi, \mathbf{0})\}$ where
\(s_\psi\) is the initial set for \(\psi\) and \(\mathbf{0}\) denotes the
all-zero vector. This represents the situation in which we
have started to evaluate \(\psi\) and no newly bound variable has been
assigned a value.

The set of accepting state sets is defined as
\begin{equation*}
   A_\varphi = \{ S\in S_\varphi \mid S \cap (A_\psi \times \{\mathbf{1}\}) \ne \emptyset \}
\end{equation*}
where \(A_\psi\) is the set of the accepting states for \(\psi\) and
\(\mathbf{1}\) denotes the all-one vector. In other words a state set
\(S\in S_\varphi\) is accepting, if and only if it contains an
accepting state pair
\((s, b)\) which consists of an accepting state \(s\in
A_\psi\) for \(\psi\) and \(b=\mathbf{1}\) meaning that
all bound variables have been assigned value. Such state pair \((s,
b)\)
corresponds to an extension \(\alpha'\) of \(\alpha\) which assigns
values to all bound variables and satisfies \(\psi\).

\paragraph{Forget lookup table}

Let $S$ be the input state set and $\delta$ be the context assignment for
variables in $\domPhiG$ with respect to a forget node. Note that the
context does not include decision variables related to the bound
variables, since we are evaluating $\varphi$, not $\psi$. Let $v_F$ be
the forgotten vertex and let $E_F$ be the set of all forgotten edges
in the forget node.

The entries of the forget table
\(\FORGET_{\varphi, w}(S, \delta)\) for a state set \(S\) are
determined by Algorithm~\ref{alg:fdpq}.
The algorithm uses a
subroutine called \(\getAllCons\) which determines all possible
consistent extensions \(\delta'\) of \(\delta\) which could serve as a
context assignment for evaluating \(\psi\) in the current forget node.
For the description of the algorithm, we assume a list of bound variables
\(x_1, \dots, x_k\) regardless of their sorts and whether they are object or set
variables.

\begin{algorithm}
   \caption{Procedure determining \(\FORGET_{\varphi, w}(S, \delta)\) for quantification}\label{alg:fdpq}
   \begin{algorithmic}
      \Require{Formula $\varphi \equiv (\exists x_1,\dots,x_k) \psi$}
      \Require{State set $S$}
      \Require{Context assignment $\delta$}
      \Ensure{New state set $S'$}
      \State{$S' = \emptyset$}
      \ForAll{states $(s,b)$ in $S$}
      \State{$X \gets \getAllCons(\psi, (x_1,\dots,x_k), s, \delta, b)$}
      \ForAll{extended assignments $(\delta',b')$ in $X$}
      \State{$s' \gets \FORGET_{\psi,w}(s,\delta')$}
      \State{$S' \gets S' \cup \{(s',b')\}$}
      \EndFor%
      \EndFor%
      \State\Return{$S'$}
   \end{algorithmic}
\end{algorithm}

Subroutine \(\getAllCons\) takes the following input parameters:
\begin{itemize}
\item A formula $\psi$.
\item List of existentially quantified variables $(x_1,\dots,x_k)$.
\item State $s \in S_\psi$ representing the current state of the procedure for $\varphi$.
\item Context assignment $\delta$ for $\varphi$ with respect to a forget
   node.
\item Bit vector $b$ specifying which object variables have already been assigned a value.
\end{itemize}

If \(x_i\) is an object variable, then \(b[x_i]\) denotes the bit of \(b\)
associated with \(x_i\).
The subroutine returns a set $X$ that contains extension pairs of the form $(\delta',
b')$. The first part is a consistent extension \(\delta'\) of the context
assignment $\delta$ that assigns values to the new decision variables related to the bound
variables ($x_1, \dots, x_k$) and forgotten objects. The second part is the
updated bit vector $b'$ specifying which object variables have already been
assigned a value.

The procedure constructs a sequence of sets \(X_0, \dots, X_k\) where a set
\(X_i, i=0, \dots, k\) is the set of extension pairs involving bound variables
\(x_1, \dots, x_i\). The first set in this sequence is set to \(X_0=\{(\delta,
b)\}\), the last set is then the returned set \(X=X_k\). We describe rules for
constructing \(X_i\) from \(X_{i-1}\) for \(i=1, \dots, k\). The rules depend on
the sort of the variable \(x_i\) and also on whether it is an object or set
variable. Based on that, possible assignments related to \(x_i\) and forgotten
objects are used to extend the pairs from \(X_{i-1}\).

Consider extension pair \((\delta', b')\in X_{i-1}\), we use the following rules to
determine which pairs are added to \(X_{i-1}\) based on \((\delta', b')\). We
distinguish four cases depending on the type of the variable \(x_i\).

\begin{itemize}
\item \emph{\(x_i\) is a vertex object variable.}
   \begin{itemize}
      \item We extend \(\delta'\) by setting \(\delta'([x_i=v_F])=0\) and add
         \((\delta', b')\) to \(X_i\).
      \item If \(b'[x_i]=0\), then we extend \(\delta'\) with assignment
         \(\delta'([x_i=v_F])=1\) and set \(b'[x_i]\) to \(1\). Pair 
         \((\delta', b')\) is then added to \(X_i\).
   \end{itemize}
\item \emph{\(x_i\) is an edge object variable.}
   \begin{itemize}
      \item We extend \(\delta'\) by setting \(\delta'([x_i=e])=0\) for all
         \(e\in E_F\) and add \((\delta', b')\) to \(X_i\).
      \item If \(b'[x_i]=0\), then we set \(b'[x_i]\) to \(1\). For every
         edge \(e\in E_F\), we extend \(\delta'\) by setting
         \(\delta'([x_i=e])=1\) and \(\delta'([x_i=e'])=0\) for every \(e'\in
         E_F\setminus\{e\}\), then we add \((\delta', b')\) to \(X_i\).
   \end{itemize}
\item \emph{\(x_i\) is a vertex set variable.}
   \begin{itemize}
      \item We extend \(\delta'\) by setting \(\delta'([v_F\in x_i])=0\) and add
         \((\delta', b')\) to \(X_i\).
      \item We extend \(\delta'\) with assignment \(\delta'([v_F\in x_i])=1\)
         and add \((\delta', b')\) to \(X_i\).
   \end{itemize}
\item \emph{\(x_i\) is an edge set variable.}
   \begin{itemize}
      \item For every subset \(E_F'\subseteq E_F\) we define an extension of
         \(\delta'\) which sets \(\delta'([e\in x_i])=1\) if and only if \(e\in
         E_F'\) and sets \(\delta'([e\in x_i])=0\) otherwise. Pair \((\delta',
         b')\) is then added to \(X_i\).
   \end{itemize}
\end{itemize}

Algorithm~\ref{alg:fdpq} processes all state pairs \((s, b)\) from the state set
\(S\). For each of them, the algorithm first constructs the set of all
consistent assignments by the application of the \(\getAllCons\) procedure
described above. The assignment \(\delta'\) from each extension pair \((\delta',
b')\) from \(X\) is then applied to state \(s\) by the forget procedure for
\(\psi\). The resulting state \(s'\) is then put into \(S'\) in a pair with the
updated bit vector \(b'\).

\paragraph{Join lookup table}

Let $S_L$ and $S_R$ be the state sets of the child nodes of the current join
node. The join procedure applies the join procedure for \(\psi\) on every pair
of states from \(S_L\) and \(S_R\). In addition, corresponding bit vectors
\(b_L\) and \(b_R\) need to be combined as well to reflect the assignment status
of each bound variable. We use bitwise conjunction (\(\land\)) and disjunction
(\(\lor\)) on the bit vectors to perform the appropriate combination. In
particular, if \(b_L\land b_R\) contains a non-zero value, it means that the
extensions of the assignments leading to states \(s_L\) and \(s_R\) both assign
a value to the same bound object variable. These assignments cannot be combined and
this combination is excluded in Algorithm~\ref{alg:jdpq} which determines the
entry of the join table.

\begin{algorithm}
   \caption{Procedure determining \(\JOIN_{\varphi, w}(s_L,
s_R)\) for quantification}\label{alg:jdpq}
   \begin{algorithmic}
      \Require{Formula $\varphi \equiv  (\exists x_1,\dots,x_k) \psi$}
      \Require{State $S_L$}
      \Require{State $S_R$}
      \Ensure{New state $S'$}
      \State{$S' = \emptyset$}
      \ForAll{state pairs $(s_L, b_L)$ in $S_L$}
      \ForAll{state pairs $(s_R, b_R)$ in $S_R$}
      \If{$b_L \wedge b_R = \mathbf{0}$}
      \State{$S' \gets S' \cup \{(\JOIN_{\psi, w}(s_L, s_R), b_L \vee b_R)\}$}
      \EndIf%
      \EndFor%
      \EndFor%
   \end{algorithmic}
\end{algorithm}

\paragraph{Correctness}

Recall that we consider undirected graph \(G=(V, E)\) with a nice tree
decomposition \(T=(V_T, E_T, \ell, r)\) of width \(w\) and a good vertex
coloring \(c\). Let $\varphi \equiv  (\exists x_1,\dots,x_k) \psi$. Assume
without loss of generality that first $k_o$ variables are object variables and
the rest are set variables. Let us consider an assignment $\alpha$ that assigns
values to the variables of \(\varphi\). Let $\Sigma(p, \alpha)$ denote the state
set which is assigned to decomposition node $p\in V_T$ in the decision procedure
for $\varphi$ on assignment $\alpha$ (more precisely, on the boolean assignment
\(\delta\) that represents \(\alpha\) using variables from \(\domPhiG\)). We
will consider extensions \(\beta\) to \(\alpha\) which assign values to all
free variables of \(\psi\), including variables
\(x_1, \dots, x_k\) that are bound in \(\varphi\). We shall use \(\sigma(p,
\beta)\) to denote the state assigned to the decomposition node \(p\) when
evaluating \(\psi\) on assignment \(\beta\) (more precisely, the corresponding
\(\mathcal{D}_{\psi, G}\)-assignment \(\delta'\)).
In addition, let us define bit vector \(B(p, \beta)\)
for each decomposition node \(p\) and extension \(\beta\) of \(\alpha\).
The bit vector has length \(k_o\) and a bit associated with each object variable
\(x_1, \dots, k_o\) defined as follows:

\begin{equation*}
   B(p, \beta)[x_i]=
   \begin{cases}
      1 & \text{object \(\beta(x_i)\) has been forgotten in the subtree rooted at \(p\)}\\
      0 & \text{otherwise}
      \end{cases}
\end{equation*}

In particular, \(B(p, \beta)=\mathbf{0}\) for all leaves \(p\) and \(B(r,
\beta)=\mathbf{1}\) for the root \(r\). Note also that if two extensions
\(\beta_1\) and \(\beta_2\) differ only on objects that have not been forgotten
in the subtree rooted at \(p\), then \(B(p, \beta_1)=B(p, \beta_2)\) and also
\(\sigma(p, \beta_1)=\sigma(p, \beta_2)\).

We now claim that the following equivalence holds:
For each \(p\in V_T\) and each state pair \((s, b)\) we have that \((s, b)\in
\Sigma(p, \alpha)\) if and only if there exists an extension \(\beta\) of
\(\alpha\) such that \(s=\sigma(p, \beta)\) and \(b=B(p, \beta)\). We can prove
this equivalence by induction, it holds for the leaves by definition of
the initial state set \(s_\varphi\) and the induction step follows from the
design of the algorithms~\ref{alg:fdpq} and~\ref{alg:jdpq}.
Applied on the root, it means that the state set \(\Sigma(r, \alpha)\) belongs
to \(A_\varphi\) if and only if there exists an extension \(\beta\) of
\(\alpha\) which satisfies \(\psi\). This implies the correctness of the procedure.

\subsection{Handling Inconsistent Assignments}\label{ssec:inconsistent-assign}

Decision procedure described in Section~\ref{sectionstate} relied on consistency
of the assignment of the decision variables. The purpose of this section is to
augment the procedure with the check of consistency. Assume an undirected graph
\(G=(V, E)\) and its nice tree decomposition \(T=(V_T, E_T, \ell, r)\) of width
\(w\). Consider also a \MSO{} formula \(\varphi\) and an assignment \(\alpha\)
to the variables of \(\varphi\). The fact that we used the representation of
\(\alpha\) in form of an assignment \(\delta\) to the decision variables in
\(\domPhiG\) will be important for our construction of a decision diagram based
on the decision procedure from Section~\ref{sectionstate}, but it was not a
necessity just for the description of this procedure. We could have described
the decision procedure directly using \(\alpha\) and, in fact, at several places
we used the assumption that \(\delta\) represents some \(\phi\)-assignment
\(\alpha\). However, our decision diagram needs to take care of assignments
to the decision variables that are not consistent. If \(\delta\) is an assignment
to the decision variables which gives zero or more than one value to a single
\MSO{} variable, the decision procedure (and later also the decision diagram)
should detect this situation and reject.

We have already needed to handle consistency in the case of existential
quantification and we shall take a similar approach here. In particular, we will
design a state space and suitable state tables for the main decision procedure,
such that the procedure returns \(1\) for a consistent assignment and returns
\(0\) for an inconsistent assignment. The idea is then to take the conjunction
of the consistency checking state space and the state space used to evaluate
$\varphi$. This will lead to a consistency-checked evaluation of $\varphi$.

\subsubsection{State Space For Assignment Consistency Checking}

By the definition of consistency, we will need to verify that exactly one of the
decision variables associated with a single object \MSO{} variable is satisfied.
We will use a special state \(\bot\) to propagate the fact that an inconsistency
was detected up to the root. The procedure will transit to this state once an
inconsistency in the assignment is detected. Other states will store a bit
vector akin to what we used for quantification. This bit vector will represent
which object variables have been assigned a value. If a variable is assigned a
value for the second time, we set the state to $\bot$. Once we reach the root,
we check whether all variables have been assigned exactly one value by checking
that the state is not $\bot$ and that the bit vector is an all-one vector.

\paragraph{State space definition}

We will now define a state space and join and forget tables which can be
combined with the full decision procedure to have a consistency checking
algorithm. The algorithm accepts if an assignment is consistent and rejects
otherwise.

The state space is defined as $S_C := {\{0,1\}}^{\kvo + \keo} \cup
\{\bot\}$, where $\kvo$ denotes the numbers of free vertex object
variables, \(\keo\) denotes the number of free edge object variables and
$\bot$ is the special new state representing inconsistency.

The initial state is defined as $s_C = \mathbf{0}$ and the set of accepting
states is set to $A_C = \{ \mathbf{1} \}$ which represents the fact that all
object variables must have a value and inconsistency should not have been
detected before.

\paragraph{Forget lookup table}

Let $s$ be the input state and $\delta$ be a context assignment of a forget
node. Let $v_F$ be the forgotten vertex and let $E_F$ be the set of all
forgotten edges in the forget node. We define the resulting state \(s'\) as follows.
If $s$ is $\bot$, then $s'=\bot$. Otherwise, \(s\in{\{0, 1\}}^{\kvo+\keo}\) and
we have to check new assignments to the object variables.

We initialize \(s'=s\). Then we loop over all assignments to decision variables
associated with object variables. For every assignment \(\delta([x=v_F])=1\) (of
a vertex object variable \(x\)) or \(\delta([x=e])=1\) for some \(e\in E_F\) (of
an edge object variable \(x\)) we proceed in the same way. If \(s'[x]=1\), then
\(s'=\bot\) and the procedure ends. Otherwise, we set \(s'[x]=1\) and go on
checking the next assignment.

Assume that for some edge object variable and two different
edges \(e_1, e_2\in E_F\) we have \(\delta([x=e_1])=\delta([x=e_2])=1\). If
\(s[x]=0\), then \(s'[x]\) is initially set to \(0\), then it is set to \(1\)
when processing the first assignment \(\delta([x=e_1])=1\) and then \(s'\) is
set to \(\bot\) when processing the second assignment \(\delta([x=e_2])=1\).

\paragraph{Join lookup table}

Let $s_L$ and $s_R$ be the states assigned to the left child and the right child node.
If any of these states is $\bot$, then the result is $\bot$. Otherwise, \(s_L,
s_R\in{\{0, 1\}}^{\kvo+\keo}\).
If \(s_L\land s_R=\mathbf{0}\), the resulting state is \(s_L\lor s_R\),
otherwise the resulting state is \(\bot\) (\(\land\) and \(\lor\) denote
bitwise conjunction and disjunction respectively).

\paragraph{Correctness}

In this case, correctness refers to the fact that the full decision procedure
with state space $S_C$ correctly determines whether an assignment is consistent.

Checking if exactly one value is assigned to each object variable is handled by
the bit vector of length $\kvo + \keo$ and the definition of accepting states.
The bit vector serves as a counter for how many times each object variable was
assigned a value. If any of the counters would reach \(2\), the state is set to
$\bot$ and this is then propagated to the root resulting in the assignment not
being accepted. Counting is handled in the leaf nodes by setting the counters to
\(0\) in the initial state. In the forget nodes the counters increase from \(0\)
to \(1\) for each satisfied decision variable assigning a value to an object
variable, if the counter is \(1\), the state changes to \(\bot\) instead of
increasing the counter value to \(2\). In the join nodes, counting is handled by
summing two bit vectors that represent assignment counts in the left subtree and
right subtree respectively. Using bitwise conjunction, the procedure checks if
the sum would overflow to two in which case the state is set to $\bot$ instead.
Finally, the definition of accepting states ensures that all counters must be set
to one in the root, which means that all variables have been assigned a value.

\subsubsection{Using Product Space To Handle Consistency}

Let $\varphi$ be an \MSO{} formula. The decision procedure with state space
$S_\varphi$ does not have well defined behavior for inconsistent assignments.
Now consider the construction that we did for conjunction in Section~\ref{sectionconj}
applied to $S_\varphi$ and $S_{C}$.

If we use the decision procedure for this space, then inconsistent assignments
will evaluate to a pair $(s_L, s_R)$ for some $s_L \in S_\varphi$ and a
non-accepting $s_R \in S_C$. By the second state being non-accepting for $S_C$,
the state $(s_L, s_R)$ is non-accepting for the product of $S_\varphi$ and
$S_C$. On the other hand, consistent assignments will evaluate to some $(s_L,
s_R)$, where $s_R$ is accepting and so whether $(s_L, s_R)$ is an accepting
state in the product space purely depends on whether $s_L$ is an accepting state
in $S_\varphi$. By the correctness of the decision procedure and by the
assignment being consistent, we will get the correct answer.


\section{Upper Bound For SDDs Parameterized By Treewidth}\label{sec:sdd-ub}

We are now ready to describe the construction of an SDD representing models of a
given \MSO{} formula \(\varphi\) interpreted over a given undirected graph
\(G\). The construction is based on the decision algorithm described in
Section~\ref{sec:decproc} and, as a consequence, we shall get the proof of
Theorem~\ref{thm:sdd-ub}.

The construction starts with precomputing the state tables for the dynamic
programming algorithm evaluating \(\varphi\) as described in
Section~\ref{sec:decproc}. The construction is recursive based on the structure
of \(\varphi\), starting with atomic subformulas and then combining the state
tables for subformulas by applying the negation, conjunction, or existential
quantification. Using these state tables, we
will then design an SDD whose structure is based on the tree decomposition of
\(G\). The v-tree respected by the SDD will be built as part of the
construction. The definition of the v-tree will include additional dummy
variables which will simplify the construction of the SDD\@. A similar approach
was used by \citet{henriksen1995mona} for constructing OBDDs (which, however,
served a different purpose for a different type of monadic second-order logic).

For the rest of this section, let us fix a \MSO{} formula \(\varphi\) which will
be interpreted over an undirected graph $G=(V, E)$ with
tree decomposition \(T=(V_T, E_T, \ell, r)\) of width \(w\), and a good vertex
coloring \(c\).  Let $S'_\varphi$ be the state space that is the product of state
space \(S_\varphi\) for $\varphi$ and the consistency checking state space
\(S_C\) (see Section~\ref{ssec:inconsistent-assign} for more details). We will
call \(S'_\varphi\) the \emph{extended state space} (as it is extended with
respect to $S_\varphi$). Let $s'_\varphi$ be the initial state of the extended
state space. Let $A'_\varphi$ be the set of accepting states in the extended
state space. Let $\FORGET'_{\varphi, w}$ and $\JOIN'_{\varphi, w}$ be the
precomputed tables for the extended state space.

The result of the construction will be an SDD \(\Delta_\varphi\) which computes
\(\funcPhiG\). The structure of \(\Delta_\varphi\) follows the structure of
\(T\) and utilizes the state spaces. Value of \(\Delta_\varphi\)on an
inconsistent \(\domPhiG\)-assignment \(\delta\) is false. If \(\delta\) is a
consistent \(\domPhiG\)-assignment, it represents a \(\varphi\)-assignment
\(\alpha\) and the value of \(\Delta_\varphi\) on \(\delta\) is true if and only
if \(\alpha\) is a model of \(\varphi\).

The construction of \(\Delta_\varphi\) and the respected v-tree follows the full
decision procedure described in Algorithm~\ref{alg:full} executed with the
extended state space on a \(\domPhiG\)-assignment \(\delta\). We will denote the
state assigned to a decomposition node \(p\in V_T\) as \(f_p(\delta)\). In each
node \(p\) of \(V_T\), we will need to store an SDD for each possible state from
the extended state space. This intermediate step of the construction will be
represented as a state-SDD mapping.

\begin{definition}[State-SDD mapping]
   A \textbf{state-SDD mapping with state space $X$} is a mapping from $X$ to
   SDDs. We say that the mapping respects v-tree $u$, if all SDDs in the image
   of the mapping respect v-tree $u$.
\end{definition}

For each decomposition node \(p\in V_T\), we will define a state-SDD mapping
$g_p$ such that
\begin{equation}\label{eq:state-sdd}
\text{$g_p(s)(\delta)$ is true, if and only if $f_p(\delta) = s$.}
\end{equation}
Since $f_p$ evaluates to exactly one state, exactly one of SDDs given by
$g_p$ will evaluate to true for any $\delta$. This enables us to use the SDDs
given by $g_p$ as primes in a node of a larger SDD\@.

\begin{rem}
   In the notation above, $g_p(s)$ is an SDD because we evaluated state-SDD
   mapping on a state. Denote the SDD as $h$. The SDD $h$ is then evaluated for
   decision variable assignment $\delta$ which we denote as $h(\delta)$. The
   notation $g_p(s)(\delta)$ then means that we use the State-SDD mapping to get
   a single SDD and then evaluate it to get either true or false.
\end{rem}

Let us note that the size of the SDD is bounded in terms of the treewidth \(w\)
of \(G\) and the size of the extended state space \(S'_\varphi\). Both these
values depend only on \(w\) and \(|\varphi|\) and are thus parameterized
constant. The size of \(S'_\varphi\) depends on the structure of the formula and
it would be hard to make a reasonable upper bound without any assumption on the
structure of \(\varphi\). We will give an upper bound on \(|S'_\varphi|\) for
formulas in prenex form in Section~\ref{sec73}. The upper bound has form of a
tower of powers of \(2\) whose height depends on the number of quantifier
alternations.

\subsection{The Construction}%
\label{seccon}

The construction proceeds by induction on the structure of the nice tree
decomposition $T$. For each node \(p\in V_T\), we define $g_p$ and we also keep
track of a v-tree \(t_p\) respected by $g_p$. We will also prove
equivalence~\eqref{eq:state-sdd}
by showing that $g_p(s)(\delta)\iff f_p(\delta) = s$ for
every state $s$ and assignment $\delta$.

When specifying v-trees, we will use a name of a decision variable to signify a
single-node v-tree whose single leaf node is labeled by that variable. We will also use
$\node(A, B)$ to denote a v-tree that consists of an internal node, with left
subtree being $A$ and the right subtree being $B$. For example, $\node(d_1,
\node(d_2, d_3))$ denotes a v-tree with three decision variables $d_1$, $d_2$,
and $d_3$.

The description of the construction is split into several steps. We will start
with an auxiliary construction related to the context assignments of forget
nodes in Section~\ref{sec:con1}. We will then describe the construction of SDDs
representing the state tables in Section~\ref{sec:con2}. The recursive construction
of the SDD is then described in sections~\ref{sec:con3}-\ref{sec:con6}. The
construction is finalized by  converting a state-SDD mapping to an SDD described
in Section~\ref{sec:con7}.

\subsubsection{Context Assignments As SDDs}%
\label{sec:con1}

The context, defined in Section~\ref{ssec:dec-context}, represents a collection
of decision variables related to a single decomposition node. As we have shown
in Lemma~\ref{lemma:soc}, the context contains at most $|\varphi|\cdot(w+1)$
decision variables and thus there are at most $2^{|\varphi|\cdot(w+1)}$
different context assignments. We will define a state-SDD mapping, which will
assign an SDD $\Sigma_\delta$ to each context assignment $\delta\colon \CTX_{\varphi,
T}(p) \to\{0, 1\}$. SDD \(\Sigma_\delta\) is defined over the decision
context variables \(\CTX_{\varphi, T}(p)\) and it evaluates to true only on
$\delta$, and no other context assignment.

Consider any ordering of variables in $\CTX_{\varphi, T}(p)$ and denote these
variables in order as $\{d_1, d_2, \dots, d_k\}$. The v-tree will be defined as the
right linear v-tree for this ordering that is
\begin{equation*}
   \node(d_1, \node(d_2, \dots \node(d_{k-2}, \node(d_{k-1}, d_{k}))\dots ))
\end{equation*}

First, we will design an SDD for a fixed context assignment $\delta \colon
\CTX_{\varphi, G}(p) \to\{0, 1\}$. The SDD for context assignment $\delta$,
denoted as $\Sigma_\delta$, is defined as follows (see Figure~\ref{fig:asigssd} for
an example):
\begin{itemize}
   \item The decomposition respecting v-tree $\node(d_{k-1}, d_{k})$, which is the base case, has form:
      \begin{itemize}
         \item $\{(d_{k-1}, d_{k}), (\neg d_{k-1}, \bot)\}$ for $\delta(d_{k-1}) = 1$ and $\delta(d_{k}) = 1$
         \item $\{(d_{k-1}, \neg d_{k}), (\neg d_{k-1}, \bot)\}$ for $\delta(d_{k-1}) = 1$ and $\delta(d_{k}) = 0$
         \item $\{(d_{k-1}, \bot), (\neg d_{k-1}, d_{k})\}$ for $\delta(d_{k-1}) = 0$ and $\delta(d_{k}) = 1$
         \item $\{(d_{k-1}, \bot), (\neg d_{k-1}, \neg d_{k})\}$ for $\delta(d_{k-1}) = 0$ and $\delta(d_{k}) = 0$
      \end{itemize}
   \item Let $i \in [k-2]$. Let $\Sigma_{i+1}$ be the SDD respecting $\node(d_{i+1}, \dots)$. The decomposition respecting v-tree $\node(d_i, \dots)$ has form:
      \begin{itemize}
         \item $\{(d_{i}, \alpha_{i+1}), (\neg d_{i}, \bot)\}$ for $\delta(d_{i}) = 1$
         \item $\{(d_{i}, \bot), (\neg d_{i}, \alpha_{i+1})\}$ for $\delta(d_{i}) = 0$
      \end{itemize}
\end{itemize}

\begin{figure}
\includegraphics[width=\linewidth]{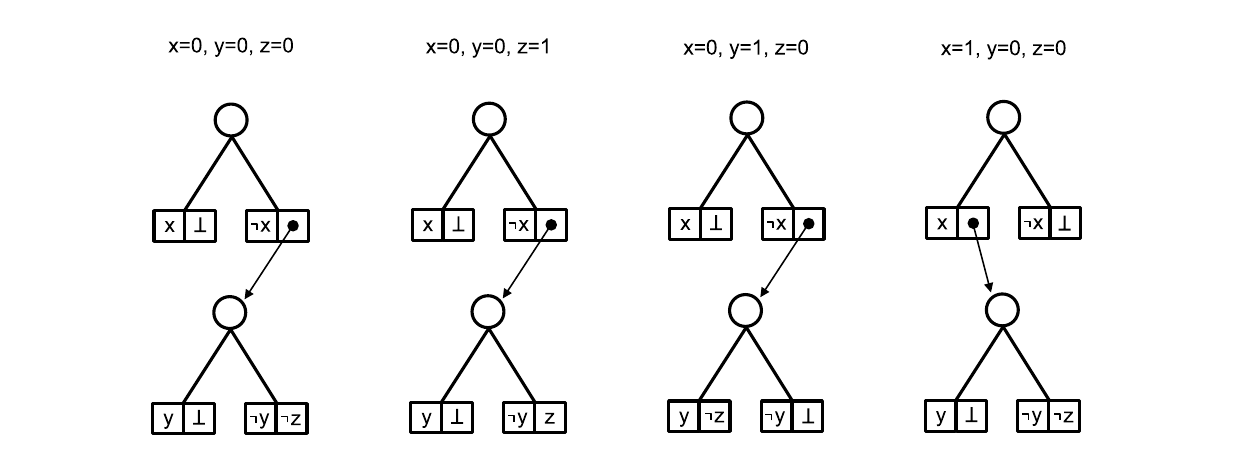}
\caption{Example of four out of eight possible assignment SDDs for variables $x$, $y$, and $z$}%
\label{fig:asigssd}
\end{figure}

If the context consists of a single decision variable \(d_1\),
then the v-tree also consists only of $d_1$ and the SDD \(\Sigma_\delta\) is $d_1$ or
$\neg d_1$ depending on \(\delta(d_1)\).

The full state-SDD mapping is then defined as $g_{C_p}(\delta) := \Sigma_\delta$.

The size of an individual SDD is at most $2\cdot \left|\CTX_{\varphi, T}(p)\right|\leq 2\cdot |\varphi|\cdot(w+1)$, since a
decomposition of size \(2\) is added at every step of the depth of the SDD is
bounded by the number of decision variables in the context (which is
\(\left|\CTX_{\varphi, T}\right|\leq |\varphi|\cdot(w+1)\) by Lemma~\ref{lemma:soc}).

There are $2^{|\varphi|\cdot(w+1)}$ possible context assignments, and thus the
total size of the state-SDD mapping is
\begin{equation}\label{eq:state-ctx-size}
   2\cdot \left|\CTX_{\varphi, p}\right|\cdot
   2^{|\varphi|\cdot(w+1)}\leq 2\cdot |\varphi|\cdot(w+1)\cdot 2^{|\varphi|\cdot(w+1)}.
\end{equation}
This is a constant for a fixed $w$ and $|\varphi|$.

\subsubsection{State Tables As SDDs}%
\label{sec:con2}

The decision procedure uses state tables to calculate states for the forget and
join nodes. In this subsection, we will derive a way to construct state-SDD
mappings which correspond to these state tables. The construction  will be used
both for the join nodes and the forget nodes, and the respective subsections
will only handle specific size calculations.

Let $F$ be a state table, which maps pairs of states from state spaces $A$ and
$B$ to state space $C$. For join nodes, state spaces $A$, $B$, and $C$ are the
state space for the given formula. For forget nodes, $A$ and $C$ are the state
spaces for the given formula and $B$ are all possible context assignments.

Let $g_A$ be the state-SDD mapping for states in $A$ and let $g_B$ be the
state-SDD mapping for states in $B$.
We will design a state-SDD mapping \(g_C\) which will associate each state $c \in C$
with an appropriate SDD\@.

Let $t_A$ and $t_B$ be v-trees for $g_A$ and $g_B$ respectively. Let $x$ be an
auxiliary variable which will not have any effect on the resulting function. The
state-SDD mapping $g_C$ that we will define will respect the following v-tree
\begin{equation}\label{eq:combined-v-tree}
   \node(t_A, \node(t_B, x))
\end{equation}

The role of the dummy variable \(x\) is to simplify the definition of \(\beta_S\) below. 
Equation~\eqref{eq:combined-v-tree}
is a correct definition of a v-tree if \(t_A\) and \(t_B\) are defined on
disjoint sets of variables. This is the case for all applications of this
construction below as we will argue in each respective subsection.

Consider an arbitrary subset $S \subseteq B$. We define an SDD for this subset denoted by $\beta_S$ as given by the following decomposition
\begin{equation*}
\beta_S := \{ (g_B(s), \top) \mid s \in S \} \cup \{ (g_B(s), \bot) \mid s \in (B \setminus S)\}
\end{equation*}

This SDD evaluates to true, if the decision procedure represented by $g_B$ ends
up in one of the states from $S$. The SDD we just defined respects v-tree
$\node(t_B, x)$.
The size of \(\beta_S\) is equal to the size of $g_B$ plus the number of the states in
$B$ for the newly defined decomposition.

Consider now a single state $c \in C$. For any $a \in A$, let $F_{a, c} := \{b
\in B \mid F(a, b) = c \}$. Set $F_{a, c}$ consists of states of $B$ whose
combination with \(a\in A\) leads to state \(c\in C\). We will use these sets to
define $\alpha_c$ which represents all assignments that evaluate to state $c$.
\begin{equation*}
   \alpha_c := \{ (g_A(a), \beta_{F_{a, c}}) \mid a \in A \}
\end{equation*}

SDD \(\alpha_c\) respects the v-tree $\node(t_A, \node(t_B, x))$. For a given
\(a\in A\), \(b\in B\), and \(c\in C\), the value of \(\alpha_c\) is true if and
only if \(F(a, b)=c\).
The size of the decomposition is $|A|$.

Finally, we define the full state-SDD mapping $g_C$ as:
\begin{equation*}
   g_C(c) := \alpha_c
\end{equation*}

\begin{figure}
   \includegraphics[width=\linewidth]{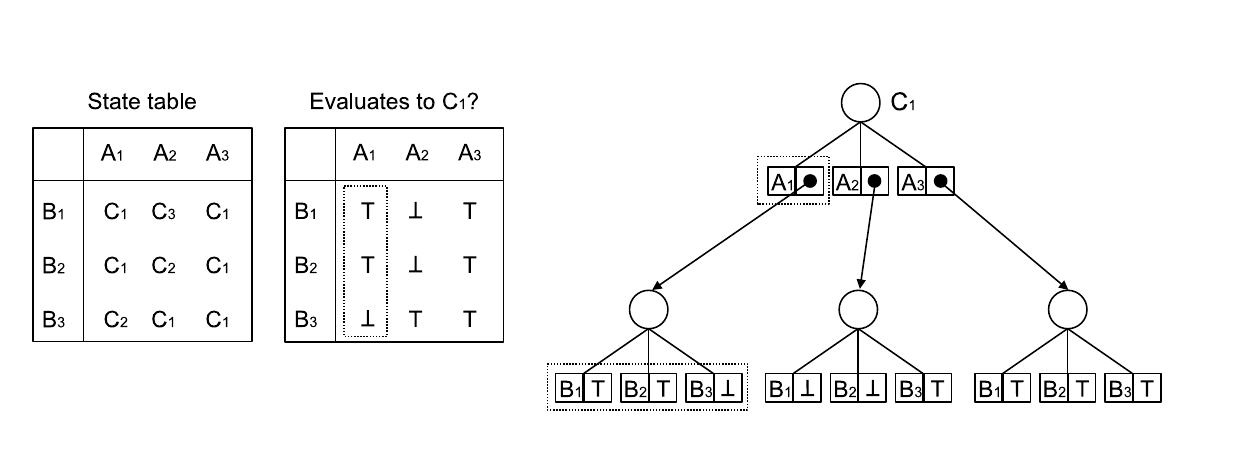}
   \caption{Example of an SSD for a single state from $C$}%
   \label{fig:butable}
\end{figure}

This state-SDD mapping respects $\node(t_A, \node(t_B, x))$.
An example of the construction of $g_C$ for a single state can be seen in Figure
\ref{fig:butable}.

The total space taken by $g_C$ is:
\begin{itemize}
\item $|A| \cdot |C|$, for $|A|$ prime-sub pairs for each decomposition in $\alpha_c$, for all $|C|$ states in $C$.
\item $|g_A|$, for utilizing all SDDs in $g_A$.
\item The size of $\beta_{F_{a, c}}$ for all $a \in A$ and $c \in C$, which is:
\begin{itemize}
\item $|B|$, for the decomposition of size $|B|$.
\item $|g_B|$, for utilizing all SDDs in $g_B$. We can reuse these for all
   $\beta_{F_{a, c}}$, so we do not multiply this term by $|A| \cdot |C|$.
\end{itemize}
\end{itemize}

This results in the total size of
\begin{equation*}
|g_C| \leq |A| \cdot |C| + |g_A| + |A| \cdot |B| \cdot |C| + |g_B|
\end{equation*}

We can be upper bound this size by combining $|A| \cdot |C|$ and $|A| \cdot |B| \cdot |C|$:
\begin{equation}\label{eq:state-cons-size}
|g_C| \leq |g_A| + |g_B| + 2 \cdot |A| \cdot |B| \cdot |C|
\end{equation}

\subsubsection{Leaf Nodes}%
\label{sec:con3}

Let $p$ be a leaf node of the tree decomposition of $G$. We will translate $p$
into the State-SDD mapping defined as follows:
\begin{itemize}
   \item $s'_\varphi$ maps to $\top$
   \item Anything else maps to $\bot$
\end{itemize}

The v-tree for a leaf node $p$ is a single node with a dummy variable $x_p$.
Both \(\top\) and \(\bot\) respect any v-tree, including this one.
Equivalence~\eqref{eq:state-sdd} holds trivially, as $f_p(\delta)=s'_\varphi$ for all assignments \(\delta\).

The size of the state-SDD mapping is $|S'_\varphi|$ (counting \(1\) for each state
\(s\in S'_\varphi\) \(s\in S'_\varphi\)).

\begin{rem}
   Since $\bot$ is not a valid prime, we use a single leaf in the decompositions
   in which the state-SDDs constructed here are used as primes of a
   decomposition.
\end{rem}

\subsubsection{Introduce Nodes}%
\label{sec:con4}

Let $p$ be an introduce node. Let $q$ be the only child of $p$.

We define $g_p$ to be $g_q$.
The v-tree $t_p$ is the same as $t_q$. Equivalence~\eqref{eq:state-sdd} follows by induction
from the fact that it holds for \(f_q\).

\subsubsection{Forget Nodes}%
\label{sec:con5}

Let $p$ be a forget node. Let $q$ be the only child of $p$. To construct
state-SDD mapping for $p$, we use the state table construction from
Section~\ref{sec:con2} with the following setting.
\begin{itemize}
   \item Set $A$ will be $S'_\varphi$ and $g_A$ is $g_q$.
   \item Set $B$ will be the set of all possible context assignments for $p$ and
      $g_B$ is given by the context construction from Section~\ref{sec:con1}.
   \item Set $C$ will be $S'_\varphi$.
\end{itemize}

Assume that \(t_A\) is the v-tree node of \(g_A=g_q\) and \(t_B\) the v-tree node for
\(g_B\). Then v-tree \(t_B\) is defined on the decision variables from
\(\CTX_{\varphi, T}(p)\), while \(t_A\) is defined on the decision variables
from the context of forget nodes below \(p\). Since each decision variable
belongs to exactly one context, we get that \(t_A\) and \(t_B\) are defined on
disjoint sets of variables and using the state table construction from
Section~\ref{sec:con2} leads to a correct v-tree \(t_p\) and by extension also to a
correct state-SDD mapping \(g_p\).

By Lemma~\ref{lemma:soc} we have that a context contains at most 
$|\varphi|\cdot(w+1)$ variables and thus \(|B|\leq 2^{|\varphi|\cdot(w+1)}\).
By estimate~\eqref{eq:state-ctx-size} in Section~\ref{sec:con1} we have that
\(|g_B|\leq 2\cdot |\varphi|\cdot(w+1)\cdot
2^{|\varphi|\cdot(w+1)}\).
Plugging these bounds to~\eqref{eq:state-cons-size} allows us to derive the
following upper bound on the size of \(g_p\).
\begin{align*}
|g_p| &\leq |g_A| + |g_B| + 2 \cdot |A| \cdot |B| \cdot |C| \\
      &\leq |g_q| + 2\cdot \left|\CTX_{\varphi, T}(p)\right| \cdot 2^{|\varphi|\cdot(w+1)} + |S'_\varphi| \cdot 2^{|\varphi|\cdot(w+1)} \cdot |S'_\varphi|\\
      &\leq |g_q| + 2\cdot \left|\CTX_{\varphi, T}(p)\right| \cdot 2^{|\varphi|\cdot(w+1)} + |S'_\varphi|^2 \cdot 2^{|\varphi|\cdot(w+1)}\\
     &\leq |g_q| + (2\cdot \left|\CTX_{\varphi, T}(p)\right| + |S'_\varphi|^2) \cdot 2^{|\varphi|\cdot(w+1)}\\
     &\leq |g_q| + (2\cdot |\varphi|\cdot(w+1) + |S'_\varphi|^2) \cdot 2^{|\varphi|\cdot(w+1)}
\end{align*}


\subsubsection{Join Nodes}%
\label{sec:con6}

Let $p$ be a join node. Let $q_L$ and $q_R$ be the left and right child node of $p$ respectively. To construct
state-SDD mapping for $p$, we use the state table construction from
Section~\ref{sec:con2} with the following setting.
\begin{itemize}
   \item The set $A$ will be $S'_\varphi$ and $g_A$ is $g_{q_L}$.
   \item The set $B$ will be $S'_\varphi$ and $g_B$ is $g_{q_R}$.
   \item The set $C$ will be $S'_\varphi$.
\end{itemize}

Assume that \(t_A\) is the v-tree node of \(g_A=g_{q_L}\) and \(t_B\) is the
v-tree node of \(g_B=g_{q_R}\). Then v-tree \(t_A\) is defined on the decision
variables from contexts of the forget nodes below \(q_L\) and \(t_B\) is defined
on the decision variables from the contexts of the forget nodes below \(q_R\).
Since each decision variable is present in exactly one context, we have that
\(t_A\) and \(t_B\) are defined on disjoint sets of variables.  Using the state
table construction from Section~\ref{sec:con2} thus leads to a correct v-tree
\(t_p\) and by extension also to a correct state-SDD mapping \(g_p\)

We can bound the size of \(g_p\) in the following way.
\begin{align*}
   |g_p| &\leq |g_A| + |g_B| + 2 \cdot |A| \cdot |B| \cdot |C| \\
         &\leq |g_q| + |g_r| + 2 \cdot |S'_\varphi|^3
\end{align*}


\subsubsection{The Root}%
\label{sec:con7}

To conclude the construction, we need to transform the state-SDD mapping $g_r$
constructed for the root $r$, into a single SDD, such that it represents the
accepting states from the state-SDD mapping.

We do this by a similar construction as when we handled the subsets of $B$ in
the state table construction.
\begin{equation*}
   \Delta_\varphi := \{ (g_r(s), \top) \mid s \in A'_\varphi \} \cup \{ (g_r(s), \bot) \mid s \in (S'_\varphi \setminus A'_\varphi)\}
\end{equation*}

This SDD respects v-tree $\node(t_r, x)$, where $t_r$ is the v-tree respected by the root state-SDD mapping \(g_r\)
and $x$ is a new dummy variable.

SDD \(\Delta_\varphi\) combines all SDDs from $g_p$ by a decomposition node of
size \(|S'_\varphi|\). We can thus estimate the size as follows.
\begin{equation*}
   |\Delta_{\varphi}| = |S'_\varphi| + |g_p|
\end{equation*}

\subsection{Complexity Of The SDD}\label{ssec:sdd-complex}

Let $p$ be a decomposition node and let us use \(q\), \(q_L\), and \(q_R\) to
denote the child nodes of \(p\) as in sections~\ref{sec:con5}
and~\ref{sec:con6}. The complexity of state-SDD mappings computed for each node
type and the SDD constructed in the root were given as follows.

\begin{itemize}
\item \textbf{Leaf nodes}: $|g_p| = |S'_\varphi|$
\item \textbf{Introduce nodes}: $|g_p| = |g_q|$
\item \textbf{Forget node \(p\)}: $|g_p| \leq |g_q| + (2\left|\CTX_{\varphi, T}|(p)\right| + |S'_\varphi|^2) \cdot 2^{|\varphi|\cdot(w+1)}$
\item \textbf{Join nodes}: $|g_p| \leq |g_{q_L}| + |g_{q_R}| + 2 \cdot |S'_\varphi|^3$
\item \textbf{The root}:  $|\Delta_{\varphi}| = |S'_\varphi| + |g_r|$
\end{itemize}

The above formulas are recursive, the size estimate for each node type includes
the estimates for the possible child nodes with an added value. If we ignore the
sizes of the state-SDDs for the child nodes, we get the actual size \(\gamma_p\)
added in node \(p\). We use \(\gamma\) to denote the additional size added by
the final step in the root (Section~\ref{sec:con7}).

\begin{itemize}
\item \textbf{Leaf node \(p\)}: $\gamma_p\leq |S'_\varphi|$
\item \textbf{Introduce node \(p\)}: $\gamma_p=0$
\item \textbf{Forget node \(p\)}: $\gamma_p\leq (2\left|\CTX_{\varphi, T}|(p)\right| + |S'_\varphi|^2) \cdot 2^{|\varphi|\cdot(w+1)}$
\item \textbf{Join node \(p\)}: $\gamma_p\leq 2|S'_\varphi|^3$
\item \textbf{The root}:  $\gamma=|S'_\varphi|$
\end{itemize}

%

By Lemma~\ref{strongnice}, there exists a nice tree decomposition of size at
most $5n$. We may use such decomposition to construct the SDD
\(\Delta_\varphi\). Summing up the contributions \(\gamma_p\) of all nodes
\(p\in V_T\) and adding \(\gamma\) gives us the estimate on the size of
\(\Delta_\varphi\). Let
us first consider the contributions of different types of nodes. We use
\(k=w+|\varphi|\) as a parameter.
\begin{itemize}
   \item\textbf{Leaf nodes}: \(\gamma_L\leq 5n|S'_\varphi|\)
   \item\textbf{Introduce nodes:} \(\gamma_I\leq 0\)
   \item\textbf{Forget nodes}: Denote \(V_F\subseteq V_T\) the set of all forget
      nodes. We have \(|V_F|\leq n\), because each vertex in \(V\) is forgotten exactly
      once by Lemma~\ref{lemma:oneforget}.
      Since each decision variable belongs to exactly one context, we
      have that \(\sum_{p\in V_F}\left|CTX_{\varphi, T}\right|\leq
      n|\varphi|\leq nk\).
      We can thus make the following estimate of the total contribution
      of the forget nodes \(\gamma_F\):
      \begin{align*}
         \gamma_F&\leq\sum_{p\in V_F}\gamma_p
         \leq\sum_{p\in V_F} (2\cdot\left|CTX_{\varphi,
      T}\right|+{|S_\varphi'|}^2)\cdot 2^{|\varphi|(w+1)}\\
                 &\leq \left(2\left(\sum_{p\in V_F}\left|\CTX_{\varphi, T}(p)\right|\right)+n{|S_\varphi'|}^2\right)\cdot  2^{|\varphi|(w+1)}\\
                 &\leq n\cdot(2k+{|S_\varphi'|}^2)\cdot 2^{k^2}
      \end{align*}
   \item\textbf{Join nodes:} \(\gamma_J\leq 5n{|S_\varphi'|}^3\)
\end{itemize}

By summing these bounds up we get the following upper bound.
\begin{equation}\label{eq:sdd-bound}
   \begin{aligned}
      |\Delta_\varphi|&\leq \gamma_L+\gamma_I+\gamma_F+\gamma_J+\gamma
   \leq 5n|S'_\varphi|+0+n(2k+{|S_\varphi'|}^2)\cdot
      2^{k^2}+5n{|S_\varphi'|}^3+|S_\varphi'|\\
                      &\leq n\left(5|S_\varphi'|+2k+{|S_\varphi'|}^2+5{|S'_\varphi|}^3+|S_\varphi'|\right)\cdot 2^{k^2}\\
                      &\leq n(12{|S_\varphi'|}^3+2k)\cdot 2^{k^2}
   \end{aligned}
\end{equation}
This bound is fixed parameter linear in $n$ with respect to parameter $k=(w +
|\varphi|)$ as long as $|S'_\varphi|$ is a parameter constant with respect to
\(k\). This is the case, as the construction of $S'_\varphi$ relied solely on
$\varphi$ and $w$ (we will give a more precise upper bound on \(|S'_\varphi|\)
for formulas in prenex form in Section~\ref{sec73}).

This construction proves Theorem~\ref{thm:sdd-ub}.

\begin{rem}
   Let us note that if \(T\) is actually a path decomposition, then there is
   only one leaf and no
   join node in \(T\). The bound~\eqref{eq:sdd-bound} then simplifies to
\(|\Delta_\varphi|\leq n(2k+3{|S_\varphi'|^2})\cdot 2^{k^2}\). It is interesting
   to compare this bound with the size of the OBDD we obtain in
   Section~\ref{sec:obdd-ub}.
\end{rem}

%

\section{Upper Bound For OBDDs Parameterized By Pathwidth}\label{sec:obdd-ub}

We will prove Theorem~\ref{thm:obdd-ub} in this section. We will use a
construction that is very similar to the one used in Section~\ref{sec:sdd-ub} to
prove Theorem~\ref{thm:sdd-ub}. We will follow the decision procedure described
in Section~\ref{sec:decproc} which is based on manipulating the state spaces 
defined in Section~\ref{sectionstate}.

Let $G=(V, E)$ be an elementary graph, let $T=(V_T, E_T, \ell, r)$ be a nice
path decomposition of width $w$, let $c$ be a good coloring of $G$ with
respect to $T$. Let $\varphi$ be an \MSO{} formula that is interpreted over
\(G\). Let $S'_\varphi$ be the state space for $\varphi$ with consistency
checking introduced in Section~\ref{ssec:inconsistent-assign} with the initial
state $s'_\varphi$ and the set of accepting states $A'_\varphi$. Let
$\FORGET'_{\varphi, w}$  be the state table for computing states for forget
nodes. Since nice path decompositions do not have join nodes, we will not need
to consider the join tables.

We will show how to derive an OBDD \(B_\varphi\) which represents the models of
\(\varphi\) in form of the boolean function \(\mathcal{F}_{\varphi, G}\). We
will start in the leaf of the path decomposition and continue upwards through
the forget nodes up to the root. Unlike the construction in
Section~\ref{sec:sdd-ub}, the OBDD will be extended at the leaf level with each
forget node, not the top level. As we follow the path decomposition \(T\), we
will be also defining the variable ordering which is respected by \(B_\varphi\).

\subsection{The Construction}\label{seccobdd}

First, we will define a structure to represent an intermediate result of the
construction. It serves the same purpose as the state-SDD
mapping introduced in Section~\ref{sec:sdd-ub}.

\begin{definition}
   A \textbf{multi-terminal binary decision diagram} (MTBDD) with terminals from
   set $S$ is a fully-labeled directed acyclic graph with nodes of the following
   two types.

   \begin{itemize}
      \item \textbf{Decision nodes}: Each decision node has out-degree \(2\) and
         unbounded in-degree. One of the outgoing edges is labeled with a \(0\)
         and the other one is labeled by a \(1\). The node itself is labeled
         with a decision variable.
      \item \textbf{Leaf nodes}: Each leaf node is labeled by an element from
         $S$.
   \end{itemize}

   Exactly one node (the \textbf{root}) has in-degree \(0\).
\end{definition}

The root node of an MTBDD can be both a decision node or a leaf node. In the latter case,
the MTBDD represents a constant function with the function value given by the
label of the leaf. A binary decision diagram is a special case of an MTBDD with
terminals from set $\{0, 1\}$.

During the construction, we associate an MTBDD \(B_p\) with every decomposition node \(p\in
V_T\). The leaves of \(B_p\) are labeled with terminals from the extended state
space \(S_\varphi'\). The decision nodes of \(B_p\) are labeled with decision
variables from the context of \(p\) and the contexts of nodes below \(p\).
For every assignment to these variables, \(B_p\) returns the state of the
decision procedure after processing node \(p\). MTBDD \(B_p\) is ordered, it
follows a specific order which respects the order of the nodes on the path
decomposition \(T\).

MTBDD \(B_p\) is defined recursively as follows.
\begin{itemize}
   \item If $p$ is a leaf node, we define MTBDD $B_p$ as a single leaf with label $s'_\varphi$.
   \item If $p$ is an introduce node with child $q$, we define $B_p = B_q$.
   \item If $p$ is a forget node with child $q$, then $B_p$ is constructed from
      $B_q$ by replacing each leaf of \(B_p\) labeled with state \(s\in
      S_\varphi'\) with an MTBDD \(B_{p, s}\) which represents the values of 
      \(\FORGET_{\varphi, w}(s, \delta)\) for all context assignments
      \(\delta\). In particular, \(B_{p, s}\) takes form of a full binary tree
      which
      respects a specific fixed order of the variables in \(\CTX_{\varphi,
      T}(p)\). Each branch of the tree corresponds to a specific context
      assignment \(\delta\) and the leaf at the end of this branch is labeled
      with \(\FORGET_{\varphi, w}(s, \delta)\). Now that we replaced all of the
      original leaves with full binary trees, we consider the new leaves and
      merge the ones with identical labels, which results in at most
      $|S'_\varphi|$ different leaves. We extend the variable
      ordering by appending all decision variables from the context of $p$. The
      order among these variable is not important to achieve a good bound. What
      matters is for the global ordering of the contexts to follow the path
      decomposition.
\end{itemize}

At the end of this construction, we are left with an MTBDD \(B_r\) representing
the states of the decision procedure after processing the root node \(r\) of the
path decomposition. To obtain the OBDD \(B_\varphi\), we iterate over the leaves
of \(B_r\). The leaves labeled \(s\) for some \(s\in A_\varphi\) are replaced
with the leaf labeled  \(1\), the rest of the leaves are replaced with the leaf
labeled \(0\).

\subsubsection{Complexity Of The OBDD}

To bound the size of \(B_\varphi\), we will bound its height and the largest number of
nodes at a single layer.

The number of decision variables of \(B_\varphi\) is bounded by \(|\varphi|\cdot
n\) by Lemma~\ref{ndvar}.  We add one to account for the layer with the leaves and
get upper bound of $(1 + |\varphi| \cdot n)$ on the height of \(B_\varphi\).
This bound can be simplified to \(2|\varphi|\cdot n\). To bound the width of
\(B_\varphi\), we first recall from Lemma~\ref{lemma:soc} that the size of a
context is at most $|\varphi|\cdot(w+1)$. The final layer of the MTBDD at every
step of the construction comprises at most $|S'_\varphi|$ leaves. However, we
are adding full binary trees in the construction for each input state in the forget nodes, so the layer
just before the leaves can have more nodes. The number of added full binary
trees is at most $|S'_\varphi|$, their height is bounded by the size of the
context, which is $|\varphi|\cdot(w+1)$. The total number of the leaves of these
full binary trees is thus bounded by $|S'_\varphi| \cdot
2^{|\varphi|\cdot(w+1)}$. This expression also bounds the width of the OBDD
\(B_\varphi\).

By multiplying the bound on the height and width of the resulting OBDD and using
the parameter \(k=w+|\varphi|\), we get
the following upper bound on the total size:
\begin{equation*}\label{eq:obdd-bound}
   |B_\varphi|\leq |S'_\varphi| \cdot 2^{|\varphi|\cdot(w+1)} \cdot (2|\varphi| \cdot n)\leq
   n \cdot 2k \cdot |S'_\varphi| \cdot 2^{k^2}
\end{equation*}

This is a parameterized linear upper bound as long as $|S'_\varphi|$ is a
constant for fixed $k$. This is the case as we argued at the end of
Section~\ref{ssec:sdd-complex}. This construction and its size bound proves Theorem~\ref{thm:obdd-ub}.

\begin{rem}
   As we remarked at the end of Section~\ref{ssec:sdd-complex}, the bound on the
   size of SDD is \(n(2k+3{|S'_\varphi|}^2)\cdot 2^{k^2}\) for the case of a
   path decomposition. In case of nontrivial formulas, it can be assumed that
   \(|S'_\varphi|\) will be at least \(k\) and then \({|S'_\varphi|}^2\) is
   bigger than \(k|S'_\varphi|\). This shows that for the path decompositions,
   construction of an OBDD may be more efficient.

   This is the consequence of a fundamental difference in the two constructions.
   In both cases, the constructions follow the path from the leaf to the root,
   but there is a difference in how the construction extends the intermediary
   structures in both constructions. In the case of SDDs, new parts are added to
   the top of the previously constructed SDDs, as a consequence, we are looking
   at the set \(F_{a, c}\) which is an image of \(c\) relative to \(b\) and may
   contain several states \(b\). In the case of OBDDs, new parts are added to
   the leaf level and we can use the fact that the states at the new leaf level
   are uniquely defined for each input state and context assignment. This allows
   us to limit the width of the OBDD by a bound that is better than in case of
   SDDs. 

   Adding the parts to the top of SDDs allows the tree structure of the SDD to
   follow the structure of the tree decomposition in the presence of join nodes.
   As a consequence, on path decompositions, the OBDD construction is possibly
   smaller as argued above.
\end{rem}

\section{A Lower Bound For OBDDs And Treewidth}\label{sec:obdd-lb}

This section is devoted to proving Theorem~\ref{thm:obdd-lb}. In particular,
we will describe a specific \MSO{} formula \(\varphi\) and  an infinite series
of undirected graphs \(G_1, G_2, \dots\) of treewidth at most \(w\) such that
any OBDD representing \(\mathcal{F}_{\varphi, G_i}\) is at least \(f(w)\cdot
n^{w/4}\) where \(n\) is the number of vertices of \(G_i\) and \(f\) is a
computable function.

We will build upon the results of \citet{razgon2014obdds} who gave a lower bound
on the size of an OBDD for a class of CNF formulas defined using a
specific class of graphs. We will show that this problem can be reformulated as
representing models of a fixed \MSO{} formula for the same class of graphs.

\subsection{\texorpdfstring{The $\CNF(G)$ Formula}{The CNF{(G)} Formula}}\label{ssec:cnfg}

In this section, we introduce the relevant definitions and results given by
\citet{razgon2014obdds}. Some of the definitions are slightly adjusted to be consistent with the terminology
used in this paper.

\begin{definition}
   Let $G=(V, E)$ be an undirected graph. We shall associate a CNF with \(G\)
   denoted as $\CNF(G)$. We associate a propositional variable \(x_v\) with
   every vertex \(v\in V\) and \(x_e\) with every edge \(e\in E\).
   In addition, we associate clause \(C_e=x_u\lor x_e\lor x_v\) with every edge
   \(e=\{u, v\}\in E\). Then we define
   \begin{equation*}
      \CNF(G) := \bigwedge_{e\in E} C_e
   \end{equation*}
\end{definition}

Note that the models of \(\CNF(G)\) correspond to vertex covers of \(G\) with
the difference that we can also use the edge variable \(x_e\) to cover \(e\)
without picking one of its endpoints.

The lower bound proved by~\citet{razgon2014obdds} uses the class of graphs
called clique trees. We use \(K_k\) to denote
the \textbf{clique} (i.e., a complete graph) on \(k\) vertices and \(T_r\) to
denote the \textbf{complete binary tree of height \(r\)} (i.e., \(T_1\) is a
single leaf and \(T_n\), \(n>1\) is composed by joining two copies of
\(T_{n-1}\) with a single root).

We define a \textbf{full product} \(G\boxtimes H\) of two undirected graphs
\(G=(V_G, E_G)\) and \(H=(V_H, E_H)\) as follows:

\begin{itemize}
   \item The vertices of $G \boxtimes H$ are $V_G \times V_H$
   \item An edge connects vertices $(v_G, v_H)$ and $(w_G, w_H)$ in $G \boxtimes
      H$, if one of the following holds:
      \begin{itemize}
         \item $v_G$ and $w_G$ are neighbours in $G$ and $v_H = w_H$
         \item $v_G = w_G$ and $v_H$ and $w_H$ are neighbours in $H$
         \item $v_G$ and $w_G$ are neighbours in $G$ and \(v_H\) and $w_H$ are neighbours in $H$
      \end{itemize}
\end{itemize}

For a given \(r\) and \(k\), we define \textbf{clique tree} $\KT_{k, r}$ as $K_k
\boxtimes T_r$.

Our lower bound relies on the following propositions proved
by~\citet{razgon2014obdds}.

\begin{theorem}\label{thm:raz1}
   The treewidth of $KT_{k, r}$ is at most $2k-1$.
\end{theorem}

\begin{theorem}\label{thm:raz2}
   The size of the OBDD, which computes $\CNF(\KT_{k, r})$ is at least $2^{rk/2}$.
\end{theorem}

\begin{theorem}\label{mainc5}
   For every $w \ge 1$, there is an infinite sequence of CNFs $F_1$, $F_2$,
   \dots, such that the primal graph of all the CNFs has treewidth at most $w$
   and the sizes of the OBDDs computing the CNFs are at least $f(w) \cdot m^{w /
   4}$, where $m$ is the number of variables in $F_i$.
\end{theorem}

The proof of Theorem~\ref{mainc5} defines $F_i := \CNF(\KT_{i,\lceil w / 2
\rceil})$. The width of the primal graph of \(F_i\) can be upper bounded using
Theorem~\ref{thm:raz1} and the  size of the OBDD can be lower bounded using
Theorem~\ref{thm:raz2}.

\subsection{\texorpdfstring{An Equivalent Problem Statement In \MSO}{An Equivalent Problem Statement In MSO2}}

\newcommand{\domKappa}{\ensuremath{\mathcal{D}_{\kappa, G}}}
\newcommand{\funcKappa}{\ensuremath{\mathcal{F}_{\kappa, G}}}
\newcommand{\domKappaI}{\ensuremath{\mathcal{D}_{\kappa, G_i}}}
\newcommand{\funcKappaI}{\ensuremath{\mathcal{F}_{\kappa, G_i}}}

We will now restate the problem introduced in Section~\ref{ssec:cnfg} using an \MSO{} formula.

\begin{definition}
   We will use the following formula with a free vertex set variable \(X_V\) and
   a free edge set variable \(X_E\) for our lower bound. We use \(V\) and \(E\)
   to denote the sets of vertices and edges of the graph over which the formula
   is interpreted to express the sort of the quantified variables.

   \begin{equation}\label{eq:kappa}
      \begin{aligned}
         \kappa(X_V, X_E) := &(\forall e \in E)(\forall u \in V)(\forall v \in V)\\
                              &\left[
                              (u \ne v) \wedge \Adj(u, e) \wedge \Adj(v, e)
                              \implies(u \in X_V) \vee (v \in X_V) \vee (e \in X_E)
                              \right]
      \end{aligned}
   \end{equation}
\end{definition}

\begin{rem}
   When writing the \MSO{} formulas in this section, we will write negations of
   equality and membership using $\ne$ and $\notin$. We also used universal
   quantification and implication in~\eqref{eq:kappa}. However, it is possible
   to rewrite~\eqref{eq:kappa} using only the symbols which were part of the
   definition of \MSO{} formulas in Section~\ref{ssec:mso} without a
   significant increase in the length of the formula.
\end{rem}

Formula \(\kappa\) interpreted over an undirected graph \(G=(V, E)\) is
connected to \(\CNF(G)\) by means of the following identity.

\begin{lemma}\label{lem:kappa-func}
   Assume we identify variable \(x_v\) of \(\CNF(G)\) with \([v\in X_V]\) from
   \(\domKappa\) for each \(v\in V\) and variable \(x_e\) of \(\CNF(G)\) with
   \([e\in X_E]\) from \(\domKappa\) for each \(e\in E\). Then
   \(\CNF(G)\equiv\funcKappa\).
\end{lemma}

\begin{proof}
   Replacing the universal quantifiers with conjunctions gives us the following
   equality for the graph~\(G\)
   \begin{align*}
      \kappa(X_V, X_E)&\equiv\bigwedge_{e\in E}\bigwedge_{u\in V}\bigwedge_{v\in V}
      \left[ (u \ne v) \wedge \Adj(u, e) \wedge \Adj(v, e) \implies(u \in X_V) \vee (v \in X_V) \vee (e \in X_E) \right]\\
                      &\equiv\bigwedge_{e=\{u, v\}\in E}[(u \in X_V) \vee (v \in X_V) \vee (e \in X_E)]
   \end{align*}
   By replacing each set relation with the corresponding variable from
   \(\domKappa\) we get that
   \begin{align*}
      \funcKappa&\equiv\bigwedge_{e=\{u, v\}\in E}\left([u \in X_V] \vee [v \in X_V] \vee [e \in X_E]\right)\\
                      &\equiv\bigwedge_{e=\{u, v\}\in E}(x_u\lor x_e\lor x_v)
                      \equiv\bigwedge_{e\in E}C_e\\
                      &\equiv\CNF(G)\text{.}
   \end{align*}
\end{proof}

Note that consistency of assignments to \(\domKappa\) variables is not an issue,
because we only use set variables and the consistency only restricts the
assignments to the object variables.

\subsection{Proof Of The Lower Bound}

We are now ready to prove Theorem~\ref{thm:obdd-lb}.

\begin{proof}[Proof of Theorem~\ref{thm:obdd-lb}]
   Following the proof of Theorem~\ref{mainc5} by~\citet{razgon2014obdds}, we
   define $G_i$ to be $\KT_{i, \lceil w / 2 \rceil}$.
   We will use \MSO{} formula \(\kappa\) introduced by~\eqref{eq:kappa}. We
   shall use \(V_i\) and \(E_i\) to denote the sets of vertices and edges of
   \(G_i\) respectively.

   The size of any OBDD representing satisfying assignments of $\CNF(G_i)$ for all
   $i$ is lower bounded by $f(w) \cdot m^{w / 4}$ by Theorem~\ref{mainc5}. The
   number of variables in \(\CNF(G_i)\) is equal to \(m=|V_i|+|E_i|\geq n\) for
   \(n=|V_i|\). It follows that the size of any OBDD representing \(\CNF(G_i)\)
   is at least $f(w) \cdot n^{w / 4}$.

   By Lemma~\ref{lem:kappa-func} we have that \(\CNF(G_i)\equiv\funcKappaI\)
   assuming the identification of the variables of \(\CNF(G_i)\) with their
   counterparts from \(\domKappaI\). Therefore an OBDD representing
   \(\funcKappaI\) must have size at least $f(w) \cdot n^{w / 4}$.
\end{proof}

\section{Additional Notes}\label{sec:notes}

In this section, we would like to add several notes that put our result into the
context of other definitions of \MSO{} and other types of graphs.

\subsection{\texorpdfstring{Alternative Definition Of \MSO{}}{Alternative Definition Of MSO2}}\label{sec71}

We would like to relate our results to two alternative definitions of \MSO{}.
\citet{COURCELLE199012} uses predicate \(\Edge\) instead of \(\Adj\). Predicate
\(\Edge\) takes three parameters, an edge object variable \(e\) and two vertex object
variables \(u\) and \(v\), the predicate is true if variables \(u\) and \(v\)
get assigned the two endpoints of the edge that is assigned to \(e\).
This predicate can be rewritten in terms of the definition of \MSO{} we use in our paper, because
\begin{equation}\label{eq:edge-pred}
   \Edge(e, u, v) \equiv \neg(u=v) \wedge \Adj(u, e) \wedge \Adj(v, e)
\end{equation}

Replacing the occurrence of $\Edge$ predicate according to~\eqref{eq:edge-pred}
increases the size of the formula by at most a constant factor which means that
the parameter \(w+|\varphi|\) also increases by constant factor as well.

In the other direction, we can define $\Adj$ using $\Edge$ as
\begin{equation*}
   \Adj(v,e) \equiv (\exists x \in V)[\Edge(e, v, x)]
\end{equation*}

Sometimes a predicate \(\Nbr\) is considered in the definition of \MSO{}.
This predicate takes two vertex object variables \(u\) and \(v\) as input and
is satisfied if the vertices assigned to \(u\) and \(v\) form an edge.
This predicate can be rewritten using the
definition of \MSO{} used in our paper as
\begin{equation*}
   \Nbr(u,v) \equiv (\exists x \in E(H))[\Adj(u,x) \wedge \Adj(v,x)]
\end{equation*}

Predicate $\Nbr$ is weaker than $\Adj$, as we cannot test adjacency of a
specific vertex and a specific edge using this predicate. This predicate is
important in $\mathit{MSO}_1$, in which we consider only vertex object variables
and vertex set variables. This predicate is sometimes also included in \MSO{},
mainly because \MSO{} can be seen as a superset of $\mathit{MSO}_1$. 

%

\subsection{\texorpdfstring{\MSO{}}{MSO2} Formulas In Prenex Form}\label{sec73}

Number of quantifier alternations is an important parameter when studying the
complexity of problems on quantified formulas. This parameter is mainly
considered on formula in the so-called prenex form. In particular, we say that a
\MSO{} formula \(\varphi\) is in \textbf{prenex form} if it has form of
\begin{equation*}
   \varphi\equiv(Q_1 x_1)(Q_2 x_2)\dots(Q_\ell x_\ell) \psi
\end{equation*}
where \(Q_i\) is a quantifier \(\exists\) or \(\forall\) for \(i=1, \dots, \ell\)
and \(\psi\) is a quantifier-free \MSO{} formula (the \textbf{matrix} of
\(\varphi\)). The number of indices \(1\leq
i<\ell\) for which \(Q_i\neq Q_{i+1}\) is then called the \textbf{number of
   quantifier alternations} of \(\varphi\).

Let $\varphi$ be an \MSO{} formula in
prenex form with $q$ quantifier alternations and with matrix $\psi$. Let
$|S_\psi|$ be the size of state space for $\psi$ as defined in
Section~\ref{sectionstate} (i.e., without consistency checking).

As we have shown in sections~\ref{sectionneg} and~\ref{sectionstateex}, applying
negation does not change the size of the state space and applying existential
quantification transforms state space of size $s$ to a state space of size $2^{s
\cdot 2^f}$, where $f$ is the number of bounded object variables.

Let us first assume for simplicity that consistency checking of the assignment
extensions is not needed when constructing the state space for quantification,
we thus do not need a bit-vector of length \(f\). In that case, applying
existential quantification would transform a state space of size $s$ to a state
space of size $2^s$ regardless of the number and types of bounded variables.

If this were the case, we can construct the state space of $\varphi$ by
processing each group of the same quantifiers at once. The size of the states
space increases exponentially with each quantifier group.
Universal quantification are handled as negated existential
quantifications. Since negations do not affect the size
of the state space, applying universal quantification is as complex as applying
existential quantification. There are $(q+1)$ groups of quantification to handle
and for each group, the bound is exponentiated. This results in the following
bound on $|S_\varphi|$, where the number of two's is $(q+1)$:
\begin{equation*}
   |S_\varphi| \leq 2^{2^{\dots^{2^{|S_\psi|}}}}
\end{equation*}

The actual constructions from sections~\ref{sectionneg} and~\ref{sectionstateex} includes
the bit-vectors for checking the consistency of the assignment extensions. The
overall height of the power tower will remain the same, however, there will be
extra terms in the expression at various levels, due to the actual bound that we
computed for quantification.


Currently, we do not know if it is possible to handle consistency in a more
efficient way which would lead to a cleaner size bound definition. We leave it
as a question for the future work.

%

\section{Conclusion}\label{sec:conclusion}

Our results are closely related to the Courcelle's
theorem~\cite{COURCELLE199012}. While the Courcelle's theorem shows that the
complexity of checking satisfiability of a given \MSO{} formula \(\varphi\) on a
given undirected graph \(G\) can be bounded by a linear parameterized bound with
parameter \(w+|\varphi|\) where \(w\) is the treewidth of \(G\), our result goes
a bit further by showing that the models can be represented using a data
structure whose size is bounded by a parameterized bound. We took our motivation
from~\citet{OKM23} who were solving a similar problem using OBDDs, but not
taking the parameterized complexity into account. In addition to OBDDs, we also
considered SDDs whose tree structure can be naturally used to capture the
structure of a tree decomposition.

In addition to the two upper bounds, the one for SDDs bounded by the treewidth
and the one for OBDDs bounded by the pathwidth, we also used an earlier lower
bound proved by~\citet{razgon2014obdds} to show that the difference in bounds is
given by the difference in the two types of decision diagrams. While SDDs have
an underlying tree structure in form of a v-tree, OBDDs must folow a linear
order of variables. This leads to the imposibility of having an upper bound on
the size of OBDDs parameterized by the treewidth.

\citet{OKM23} considered a bottom-up approach to the construction of an OBDD
representing the models of a given \MSO{} formula on a given undirected graph.
The idea was to start with the atoms and then model the logical structure of
\MSO{} by logical operations on OBDDs. Our preliminary results show that it
should be possible to do something similar with the parameterized bounds. It
turns out, however, that in the case of SDDs, we need a stronger restrictions on
the SDDs we work with. In particular, all SDDs need to satisfy that all
decision nodes respecting the same v-tree node have the same set of primes in
the decomposition. We do not know yet, if the construction can do without this
restriction and so finishing this line of research is our planned future work.

Once we have models of an \MSO{} formula compiled into an SDD or OBDD, we can
use the decision diagrams to enumerate or count the
models~\cite{darwiche2011sdd}. In addition, SDDs form a subclass of decomposable
negation normal forms (DNNF) which by~\citet{D01} allow to find an assigment
with minimum or maximum cardinality, or even enumerate such assignments.
Cardinality of an assignment equals to the number of variables set to true. This
can be used in optimization tasks. Consider for instance a \MSO{} formula
\(\varphi\) with a single vertex set variable \(S\) whose models are exactly the
vertex covers of the underlying graph. Minimum cardinality model of the SDD
\(\Delta_\varphi\) representing \(\varphi\) is then a minimum vertex cover of
the given graph.

%

\printbibliography%

\end{document}